\documentclass[twoside]{article}

\usepackage[accepted]{aistats2024}
\usepackage[round]{natbib}

\usepackage[colorlinks,citecolor=blue]{hyperref}
\usepackage{paralist}
\usepackage{wrapfig}
\usepackage{graphicx}
\usepackage{caption}
\usepackage{subcaption}
\usepackage{multirow}
\usepackage{enumitem}

\usepackage{dblfloatfix}
\usepackage{amsmath,amsfonts,bm,amsthm}
\usepackage{algorithm}
\usepackage{algpseudocode}
\def\eqref#1{equation~\ref{#1}}
\def\1{\bm{1}}

\def\rvx{{\mathbf{x}}}
\def\rvy{{\mathbf{y}}}

\DeclareMathAlphabet{\mathsfit}{\encodingdefault}{\sfdefault}{m}{sl}
\SetMathAlphabet{\mathsfit}{bold}{\encodingdefault}{\sfdefault}{bx}{n}

\def\sX{{\mathbb{X}}}

\newcommand{\E}{\mathbb{E}}

\newcommand{\R}{\mathbb{R}}

\newcommand{\norm}[1]{\left\lVert#1\right\rVert}
\newcommand{\B}{\boldsymbol}

\newcommand{\C}{\mathcal}

\theoremstyle{plain}
\newtheorem{theorem}{Theorem}
\newtheorem{lemma}{Lemma}
\theoremstyle{definition}

\newtheorem*{assumpiton*}{Assumption}

\newcommand{\modelname}{\texttt{Skinny Trees}}

\newcommand{\modelsname}{\texttt{Skinny Tree}}

\makeatletter  
\usepackage{stmaryrd}  
\newcommand{\fixed@sra}{$\vrule height 2\fontdimen22\textfont2 width 0pt\shortrightarrow$}
\newcommand{\shortarrow}[1]{%
 \mathrel{\text{\rotatebox[origin=c]{\numexpr#1*45}{\fixed@sra}}}
}
\makeatother

\begin{document}



\twocolumn[

% \aistatstitle{An End-to-end Optimization Framework for Feature Selection in Tree Ensembles}
\aistatstitle{End-to-end Feature Selection Approach for Learning \modelname}

\aistatsauthor{ Shibal Ibrahim \And Kayhan Behdin \And Rahul Mazumder}

\aistatsaddress{ MIT \\ Cambridge, MA, USA \\shibal@mit.edu \And MIT \\ Cambridge, MA, USA \\ behdink@mit.edu \And  MIT \\ Cambridge, MA, USA \\ rahulmaz@mit.edu}
]

\begin{abstract}
We propose a new optimization-based approach for feature selection in tree ensembles, an important problem in statistics and machine learning. Popular tree ensemble toolkits e.g., Gradient Boosted Trees and Random Forests support feature selection post-training based on feature importance scores, while very popular, they are known to have drawbacks.
We propose \modelname: an end-to-end toolkit for feature selection in tree ensembles where we train a tree ensemble while controlling the number of selected features.
Our optimization-based approach learns an ensemble of differentiable trees, and simultaneously performs feature selection using a grouped $\ell_0$-regularizer. 
We use first-order methods for optimization and present convergence guarantees for our approach.
We use a dense-to-sparse regularization scheduling scheme that can lead to more expressive and sparser tree ensembles. 
On 15 synthetic and real-world datasets, \modelname~can achieve $1.5\!\times\! -~620~\!\times\!$ feature compression rates, leading up to $10\times$ faster inference over dense trees, without any loss in performance.
\modelname~lead to superior feature selection than many existing toolkits e.g., in terms of AUC performance for 25\% feature budget, \modelname~outperforms LightGBM by $10.2\%$ (up to $37.7\%$), 
and Random Forests by $3\%$ (up to $12.5\%$).
\end{abstract}

\section{INTRODUCTION}
Decision trees have been popular in various machine learning applications  \citep{Erdman2016, xgboost2016} for their competitive performance, interpretability, robustness to outliers, and ease of tuning \citep{Hastie2009}.
Many scalable toolkits for learning tree ensembles have been developed \citep{Breiman2001,xgboost2016,lightgbm2017,catboost2018}.
While these toolkits are excellent for building tree ensembles, they do not allow for feature selection during the training process. 

Feature selection is a fundamental problem in machine learning and statistics and has widespread usage across various real-world tasks \citep{Wulfkuhle2003,Cai2018,Li2017}.
Popular tree ensemble toolkits only allow selecting  informative features post-training based on feature importance scores, which are known to have drawbacks\footnote{They are found to hurt performance in (a) settings where number of samples are smaller than features and (b) settings with correlated features (see Sec. \ref{sec:simulation}).} in the context of feature selection~\citep{Strobl2007,Boulesteix2011,Zhou2021}.
Recently, there has been some work on optimization-based approaches for feature selection in trees. For example, \citep{Zharmagambetov2020} consider oblique decision trees (hyperplane splits at every node), and use $\ell_1$-penalization to encourage coefficient sparsity at every node of the tree.
This achieves node-level feature selection and does not appear to be well-suited for tree-level or ensemble-level feature selection (See Sec.~\ref{sec:comparison-with-tao}).
\citet{Liu2021} proposed ControlBurn that considers a lasso based regularizer for feature selection on a pre-trained forest.
This can be viewed as a two-stage procedure (unlike an end-to-end training procedure we propose here), where one performs feature selection after training a tree ensemble with all original features,  
While these methods serve as promising candidates for feature selection, these works highlight that identifying relevant features \textit{while} learning compact trees remains an open challenge---an avenue we address in this work. 

In many real world problems there are costs associated with features reflecting time, money, and other costs related to the procurement of data~\citep{Min2014,Zhang2010}.
In this context, one would like to collect a compact set of features to reduce experimental costs. Additionally, selecting a compact set of relevant features can lead to enhanced interpretability \citep{Ribeiro2016}, faster inference, decreased memory footprint, and even improved model generalization on unseen data \citep{Chandrashekar2014}.

In this paper, we propose an end-to-end optimization framework for feature selection in tree ensembles where we jointly learn the trees and the relevant features in \emph{one} shot. 
Our framework is based on differentiable (a.k.a. soft) tree ensembles \citep{Jordan1994,Kontschieder2015,Hazimeh2020b, Ibrahim2022} where tree ensembles are learnt based on differentiable programming. These works, however, do not address feature selection in trees which is our focus. We use a sparsity-inducing penalty (based on the group $\ell_0-\ell_2$-penalty) to encourage feature selection. 
While group $\ell_0-\ell_2$ penalty has been found to be useful in recent work on high-dimensional linear models~\citep{Hazimeh2021Group} and additive models~\citep{Hazimeh2021Group,Ibrahim2021}, 
their adaptation to tree ensemble presents unique challenges.
To obtain high-quality models with good generalization-and-sparsity tradeoffs, we need to pay special attention to dense-to-sparse training, which differs from sparse-to-dense training employed in linear/additive models above.
We demonstrate that our end-to-end learning approach leads to better accuracy and feature sparsity tradeoffs.

\noindent\textbf{Contributions.}
We propose a novel end-to-end optimization-based framework for feature selection in tree ensembles. 
We summarize our contributions in the paper as follows:
\begin{itemize}[noitemsep,topsep=0pt,parsep=0pt,partopsep=0pt, leftmargin=*]
    \item We propose a joint optimization approach, where we simultaneously perform 
    feature selection and tree ensemble learning. Our joint training approach is different from post-training feature selection in trees. Our approach learns (differentiable) tree ensembles with a budget on feature sparsity where the latter is achieved via a group $\ell_0$-based regularizer.
    \item Our algorithmic workhorse is based on proximal mini-batch gradient descent (GD). We also discuss the convergence properties of our approach in the context of a nonconvex and nonsmooth objective.
    When our first-order optimization methods are used with \emph{dense-to-sparse} scheduling of regularization parameter, we obtain tree ensembles with better accuracy and feature-sparsity tradeoffs.  
    \item We introduce a new toolkit: \modelname.
    We consider 15 synthetic and real-world datasets, showing that \modelname can lead to superior feature selection and test AUC compared to popular toolkits. In particular, for 25\% feature budget, \modelname~outperforms LightGBM by $10.2\%$ (up to $37.7\%$), XGBoost by $3.1\%$ (up to $17.4\%$), and Random Forests by $3\%$ (up to $12.5\%$) in test AUC.
\end{itemize}

\section{RELATED WORK}
\noindent\textbf{Trees.}
A popular and effective method to construct a single decision tree is based on recursive greedy partitioning (e.g., CART)~\citep{Hastie2009}. Popular methods to construct tree ensembles are based on bagging \citep{Breiman2001}, sequential methods like Boosting~\citep{Hastie2009}, etc.
These have led to various popular toolkits, e.g., Random Forests \citep{Breiman2001}, Gradient Boosted Trees \citep{xgboost2016,lightgbm2017,catboost2018}.
Another line of work that is most related to our current approach uses 
differentiable (or smooth) approximations of indicator functions at the split 
nodes~\citep{Jordan1994}. First-order methods (e.g, SGD) are used for 
end-to-end differentiable training of tree ensembles~\citep{Kontschieder2015, Hazimeh2020b,Ibrahim2022}.
Joint training of soft tree ensembles often results in more compact representations (i.e., fewer trees) compared to 
boosting-based procedures~\citep{Hazimeh2020b}.
Despite the success and usefulness of both these approaches, to our knowledge, there is no prior work that performs simultaneous training and feature selection---a void we seek to fill in this work. 
We next summarize some popular feature selection methods. 

\noindent\textbf{Feature selection.} We review some prior work on feature selection as they relate to our work. We group them into three major categories: 
\begin{enumerate}[noitemsep,topsep=0pt,parsep=0pt,partopsep=0pt, leftmargin=*]
    \item Filter methods attempt to remove irrelevant features before
    model training.
These methods perform feature screening based on statistical measures that quantify
feature-specific relevance scores~
\citep{Battiti1994,Peng2005,Estevez2009,Song2007,Song2012,Chen2017}. These scores consider the marginal effect of a feature over the joint effect of feature interactions. 
\item Wrapper methods \citep{Kohavi1997,Stein2005,Zhu2007,Reunanen2003,Allen2013,Onnia2001,Kabir2010,Roy2015}
use the outcome of a model to determine the relevance of each feature.
Some of these methods require recomputing the model for a subset of features and can be computationally expensive.
This category also includes feature selection using 
feature importance 
scores of a 
pre-trained model. 
Many tree ensemble toolkits~\citep{Breiman2001,xgboost2016,lightgbm2017,catboost2018} produce feature-importance scores from a pre-trained ensemble. 
\citet{Lundberg2017} propose SHAP values as a unified measure of feature importance. \citet{Sharma2023} uses SHAP values to select a subset of features that can be useful for secondary model performance characteristics e.g., fairness, robustness etc.  
\citet{Liu2021} propose ControlBurn, which formulates an optimization problem with a Lasso-type regularizer to perform feature selection on a pre-trained forest. 
\item Embedded methods \textit{simultaneously} learn the model and the relevant subset of relevant features. Notable among these methods include $\ell_0$-based procedures \citep{Hazimeh2020L0Learn,Hazimeh2021Group,Ibrahim2021,Ibrahim2023grandslamin} and their variants based on lasso~\citep{Tibshirani1996,Ravikumar2009,Zhao2012} in the linear and additive model settings.
Some distributed and stochastic greedy methods have also been explored for subset selection~\cite[see, for example,][and references therein]{Khanna2017}.
Embedded nonlinear feature selection methods have been explored for neural networks. 
For example, \citet{Chen2021} use an active-set style algorithm for cardinality-constrained feature selection.
Other approaches include the use of (group) lasso type methods \citep{Scardapane2017,Feng2017,Vu2020,Lemhadri2021}, 
or reparameterizations of $\ell_0$-penalty with stochastic gates~\citep{Louizos2018,Yamada2020}.
\end{enumerate}

\looseness=-1 We propose an embedded approach that simultaneously performs feature selection and tree ensemble learning. 
This joint training approach can be useful for compression, efficient inference, and/or generalization.

\noindent\textbf{Organization.} The rest of the paper is organized as follows. Section \ref{sec:background} summarizes relevant preliminaries. Section \ref{sec:problem-formulation} presents a formulation for feature selection in soft tree ensembles. 
Section \ref{sec:end-to-end-sparse-learning} discusses our optimization algorithm and its convergence properties. Later, we discuss a scheduling approach that can result in better accuracy and feature sparsity tradeoffs. Sections \ref{sec:simulation} and \ref{sec:experiments} perform  experiments on a combination of 15 synthetic and real-world datasets to highlight the usefulness of our proposals.

\section{PRELIMINARIES}
\label{sec:background}
We learn a mapping $\B f: \R^p \rightarrow \R^c$, from input space $\mathcal{X} \subseteq \R^p$ to output space $\mathcal{Y} \subseteq \R^c$, where we parameterize function $\B f$ with a soft tree ensemble.
In a regression setting $c=1$, while in multi-class classification setting $c=C$, where $C$ is the number of classes.
Let $m$ be the number of trees in the ensemble and let $\B f^{j}$ be the $j$th tree in the ensemble.
We learn an additive model with the output being sum over outputs of all the trees: $\B f(\B x) = \sum_{j=1}^m \B f^{j}(\B x)$ for an input feature-vector $\B x \in \R^p$.
A summary of the notation can be found in Table \ref{tab:notation} in Supplement.

Compared to classical trees, soft trees allow for much more flexibility in catering to different loss functions \citep{Ibrahim2022}, sparse routing in Mixture of Experts \citep{Ibrahim2023comet} etc.
A soft tree is a differentiable variant of a classical decision tree, so learning can be done using gradient-based methods. 
It was proposed by \citet{Jordan1994}, and further developed by \citet{Kontschieder2015,Hazimeh2020b,Ibrahim2022} for end-to-end optimization.
Soft trees typically perform soft routing, i.e., a sample is fractionally routed to all leaves; but can be modified to do hard routing in the spirit of conditional computing~\citep{Hazimeh2020b,Ibrahim2023comet}.
Soft trees are based on hyperplane splits, where the routing decisions rely on a linear combination of the features. 
Particularly, each internal node is associated with a trainable weight vector that defines the associated hyperplane.
Formulations \citep{Ibrahim2022} for training soft tree ensembles more efficiently have been proposed, which exploit tensor structure of a tree ensemble.
A summary of soft trees and tensor formulation is in Supplement Sec. \ref{supp-sec:soft-trees}.

\section{PROBLEM FORMULATION}
\label{sec:problem-formulation}
Feature selection plays a ubiquitous role in modern statistical regression, especially when the number of predictors is large relative to the number of observations.
We describe the problem of global feature selection. 
We assume a data-generating model $p(\rvx; \rvy)$ over a $p$-dimensional space, where $\rvx \in \R^p$ is the covariate and $\rvy$ is the response.
The goal is to find the best function $\B f(\rvx)$ for predicting $\rvy$ by minimizing: 
\begin{align}
{\min}_{\B f \in \C{F}, \C{Q}} ~~\E[L(\rvy, \B f(\rvx_{\C{Q}}))]
\label{eq:general-formulation}
\end{align}
where $\C{Q} \subseteq \{1,2,\cdots,p\}$ is an unknown (learnable) subset of features of size at most $K$, $\B f$ is a learnable non-parametric function from the function class $\C{F}$, and $L: \R^p \times \R^c \rightarrow \R$ is a loss function.
The principal difficulty in solving (\ref{eq:general-formulation}) lies in the joint optimization of $(f, \C{Q})$---the number of subsets $\C{Q}$ grows exponentially with $p$.
In addition, the family of functions $\C{F}$ needs to be sufficiently flexible (here, $\C{F}$ is the class of soft tree ensembles with fixed ensemble size $m$ and depth $d$).

\begin{figure}[!b]
    \centering
    \includegraphics[width=\columnwidth]{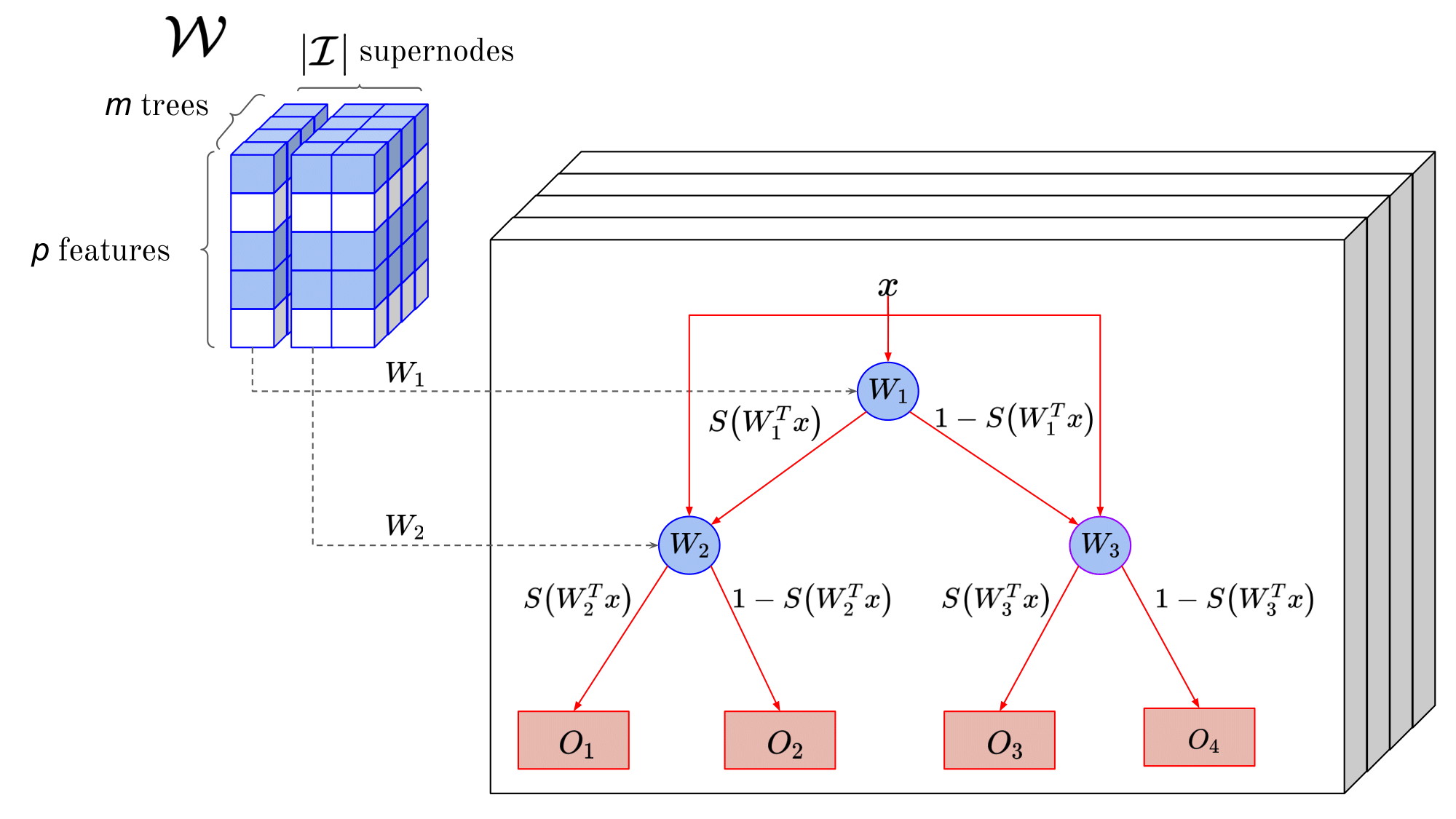}
    \caption{\looseness=-1 \emph{Illustration of \modelname. Each horizontal slice $\bm{\C{W}}_{k,:,:}$ depicts a single feature. White slices indicate features filtered out by the ensemble while training.
    Each vertical slice (along the depth of the page), $\bm{\C{W}}_{:,:,j} = \bm{W}_j$ corresponds to parameters in $j$-th (splitting) supernode (blue circles) in the ensemble, eventually producing the routing decisions. 
    The red squares depict leaf nodes. $S(\cdot)$ denotes an activation function, which can be Sigmoid \citep{Jordan1994} or Smooth-Step \citep{Hazimeh2020b}.}}
    \label{fig:sparse-tree-ensemble}
\end{figure}

In the context of tree ensembles, our goal for global feature selection is to select a small subset of features across \textit{all} the trees in the ensemble. 
More specifically, we consider the framework with response $\B y$ and prediction $\B f(\B x; \bm{\C{W}}, \bm{\C{O}})$, where the function $\B f$ is parameterized by learnable hyperplane parameters $\B{\C{W}} \in \R^{p, m, |\C{I}|}$ across all the split nodes and learnable leaf parameters $\bm{\C O}$ across all trees.
Note that $m$ is the number of trees in the ensemble and $|\bm{\C{I}}|$ represents the number of (split) nodes in each tree.
This parameterization of soft trees points to a group structure in $\bm{\C{W}}$, where the whole slice $\bm{\C{W}}_{k,:,:}$ in the tensor formulation has to be zero to maintain feature sparsity both across all split nodes in each tree and across the trees in the ensemble --- see Fig. \ref{fig:sparse-tree-ensemble}. 
This is a natural feature-wise non-overlapping group structure and allows adaptation of the grouped selection problem in linear models \citep{Bertsimas2016,Hazimeh2021Group} to soft tree ensembles.

\noindent\textbf{Mixed Integer Problem (MIP) Formulation.} Let us consider the tensor $\bm{\C{W}}$. Using binary variables to model feature selection we obtain a regularized loss function:
\begin{align}
    \min_{\bm{\C {W}}, \bm{\C O}, {\bm{z}}}&~\hat{\E}[L(\rvy, \B f( \rvx; \bm{\C {W}}, \bm{\C O})] + \lambda_0 \sum_{k \in [p]} z_k, \label{eq:mip-formulation} \\ \text{s.t.}&~||\bm{\C {W}}_{k,:,:}|| (1- z_k)=0, z_{k} \in \{0,1\}~~k\in[p], \nonumber
\end{align}
where, the binary variable $z_{k}$ controls whether the $k$-th feature is on or off via the constraint $||\bm{\C {W}}_{k,:,:}|| (1-z_k) = 0$. $\hat{\E}[L(\rvy, \B f(\rvx; \bm{\C {W}}, \bm{\C O}))] := (1/N) \sum_{n \in [N]} L( \B y_n,\B f(\B x_n; \bm{\C {W}}, \bm{\C O}))$ is the empirical loss; and $\lambda_0$ is regularization strength.
Note MIP formulations \citep{Bertsimas2019} can also be setup with classical trees under feature selection, but they would be difficult to scale beyond small problems.

\noindent\textbf{Unconstrained formulation of Problem~(\ref{eq:mip-formulation}).} 
For computation we consider a penalized version of~(\ref{eq:mip-formulation}) involving variables ($\bm{\C{W}},\bm{\C{O}}$) with the grouped $\ell_0$ (pseudo) norm encouraging feature sparsity. 
We perform end-to-end training via first-order methods (see Sec.~\ref{sec:end-to-end-sparse-learning} for details). 
It has been observed in the linear model setting that a vanilla (group) $\ell_0$ penalty may result in overfitting~\citep{Hazimeh2021Group}.
A possible way to ameliorate this problem is to include an additional ridge regularization for shrinkage \citep{mazumder2017subset,Hazimeh2021Group,Ibrahim2021}.
We consider the following group $\ell_0-\ell_2$ regularized problem:
\begin{align}
    {\min}_{\bm{\C {W}}, \bm{\C O}}&~~\hat{\E}[L(\rvy, \B f( \rvx; \bm{\C {W}}, \bm{\C O}))] \label{eq:group-l0-l2}\\
    &+ \lambda_0 {\sum}_{k \in [p]} 1[\bm{\C{W}}_{k,:,:} \neq \B 0] + (\lambda_2/m|\C{I}|) ||\bm{\C{W}}||_2^2 \nonumber
\end{align}
where $1[\cdot]$ is the indicator function, $\lambda_0 \geq 0$ controls the number of features selected, and $\lambda_2 \geq 0$ controls the amount of shrinkage of each group. We normalize $\lambda_2$ by the product $m|\bm{\C{I}}|$ for convenience in hyperparameter tuning.

\section{END-TO-END OPTIMIZATION APPROACH}
\label{sec:end-to-end-sparse-learning}
\looseness=-1 We propose a fast approximate algorithm to obtain high-quality solutions to Problem (\ref{eq:group-l0-l2}).  
We use a proximal (mini-batch) gradient-based algorithm \citep{lan2012optimal}
that involves two operations.
A vanilla mini-batch GD step is applied to all model parameters followed by 
a proximal operator applied to the hyperplane parameters $\bm{\C{W}}$. 
This sequence of operations on top of backpropagation makes the procedure simple to implement in popular ML frameworks e.g. Tensorflow~\citep{tensorflow2015}, and contributes to overall efficiency.

\subsection{Proximal mini-batch gradient descent}
We first present the proximal mini-batch GD algorithm for solving Problem (\ref{eq:group-l0-l2}) in Algorithm~\ref{algo:proximal-stochastic-gradient-descent}. We also discuss computation of the \textit{Prox} operator in line 7 of Algorithm~\ref{algo:proximal-stochastic-gradient-descent}.
\textit{Prox} finds the global minimum of the optimization problem:
\begin{align}\label{prox-defn-1}
    \B{\C{W}}^{(t)} = \text{argmin}_{\B{\C{W}}}&~~(1/2\eta) ||\B{\C{W}} - \B{\C{Z}}^{(t)}||_2^2 \nonumber \\ 
    &+  \lambda_0 {\sum}_{k\in[p]} 1[\bm{\C{W}}_{k,:,:} \neq 0]
\end{align}
where $\B{\C{Z}}^{(t)} = \B{\C{W}}^{(t-1)}-\eta\nabla_{\bm{\C{W}}}h$.
Problem~\ref{prox-defn-1} decomposes across features and a solution for the $k$-th feature can be found by a hard-thresholding operator given by:
\begin{align}
    H_{\eta\lambda_0}(\B{\C{Z}}^{(t)}_{k,:,:}) = \B{\C{Z}}^{(t)}_{k,:,:}\odot1\left[||\B{\C{Z}}^{(t)}_{k,:,:}|| \geq \sqrt{2\eta\lambda_0}\right].
    \label{eq:proximal-operator-group-l0-l2}
\end{align}

{\small{\begin{algorithm}[!b]
\begin{algorithmic}[1]
\Require Data: $X, Y$;
\Require Hyperparameters: $\lambda_0,\lambda_2$, epochs, batch-size, learning rate ($\eta$); 
\State Initialize ensemble with $m$ trees of depth $d$ ($|\C{I}|=2^d - 1$): $\bm{\C{W}}, \bm{\C{O}}$
\For {$\text{epoch}=1,2,\ldots,\text{epochs}$}
\For {$batch=1,\ldots,N/\text{batch-size}$}
\State Randomly sample a batch: $\B X_{batch}, \B Y_{batch}$ 
\State Compute gradient of loss $h$ w.r.t. $\bm{\C{O}}$, \bm{\C{W}}.
\State Update leaves: $\bm{\C{O}} \gets \bm{\C{O}} - \eta \nabla_{\bm{\C{O}}}h$
\State Update hyperplanes: $\bm{\C{W}} \gets \textit{Prox}(\bm{\C{W}} - \eta \nabla_{\bm{\C{W}}}h, \eta, \lambda_0))$
\EndFor
\EndFor
\end{algorithmic}
where $h=\hat{\E}[L(\B Y_{batch}, \B F_{batch}] + (\lambda_2/m|\C{I}|) ||{\bm{\C{W}}}||_2^2$
\caption{Proximal Mini-batch Gradient Descent for Optimizing (\ref{eq:group-l0-l2}).}
\label{algo:proximal-stochastic-gradient-descent}
\end{algorithm}}}
We use feature-wise separability for faster computation.
The cost of \textit{Prox} is of the order $\C{O}(v)$, where $v$ is the total number of hyperplane parameters being updated (i.e. $v = pm|\C{I}|$).
This cost is negligible compared to the computation of the gradients with respect to the same parameters. 
We implement the optimizer in standard deep learning
APIs. 

To our knowledge, 
our proposed approach (and algorithm) for group $\ell_0$-based nonlinear feature selection in soft tree ensembles is novel.
Note that \citet{Chen2021} considers group $\ell_0$ based cardinality constrained formulation for feature selection in neural networks. However, their active-set style approach is very different from our iterative-hard-thresholding based approach. Their toolkit also appears to be up to $900\times$ slower than our toolkit.
In the context of neural network pruning, modifications of iterative-hard-thresholding based approaches \citep{jin2016training,peste2021acdc} have appeared for individual weight pruning in neural networks which is different from feature-selection.

\subsection{Convergence Analysis of Algorithm~\ref{algo:proximal-stochastic-gradient-descent}}
In this section, we analyze the convergence properties of Algorithm~\ref{algo:proximal-stochastic-gradient-descent}. For simplicity, we assume that the outcomes are scalar, i.e. $c=1$ and consider the least squares loss (see Assumption~\ref{lproperties} below). We also analyze the full-batch algorithm, i.e., we assume $\text{batch-size}=N$. Extending our result to vector outputs and a mini-batch algorithm would require appropriate changes, and is omitted here. 
Before stating our formal results, we discuss our assumptions on the model.
\begin{enumerate}[leftmargin=*,label=\textbf{(A\arabic*})]
    \item \textit{(Activation function, $S(\cdot)$).} \label{sproperties} $S(\cdot)$ is a differentiable piece-wise polynomial function, 
    \begin{equation*}
        S(x)=\begin{cases} 1 &\mbox{if}~~~x>\theta \\
        p(x) &\mbox{if}~~~-\theta\leq x\leq\theta \\
        0 &\mbox{if}~~~x<-\theta \\
        \end{cases}
    \end{equation*}
    for some $\theta>0$ and a polynomial function $p(x)$. We assume the derivative of $S(x)$ is continuous.
    \item \label{lproperties} \textit{(Loss function, $L$).} We use the least squares loss $L(x,y)=(x-y)^2/2$.
    \item \label{o-bounded} \textit{(Solutions).} There exists a numerical constant $B>0$ only depending on the data, such that $\|\B{\C{O}}\|_2\leq B$ for all iterations in Algorithm~\ref{algo:proximal-stochastic-gradient-descent}.
\end{enumerate}
Assumption~\ref{sproperties} encompasses a general class of activation functions, and includes the smooth-step function~\citep{Hazimeh2020b} used in soft trees. The least squares loss in Assumption~\ref{lproperties} is a standard loss for regression problems. We consider the assumption on the boundedness of leaf weights in Assumption~\ref{o-bounded} to be weak as the data is bounded\footnote{Alternatively, one can add additional projection steps for leaf weights in Algorithm~\ref{algo:proximal-stochastic-gradient-descent}, ensuring the boundedness property, and remove Assumption~\ref{o-bounded}}.

Theorem~\ref{conv-thm} states our main result in this section:
\begin{theorem}\label{conv-thm}
Let $\lambda_2>0$ and suppose Assumptions~\ref{sproperties},~\ref{lproperties} and~\ref{o-bounded} hold. Then:
\begin{enumerate}[noitemsep,topsep=0pt,parsep=0pt,partopsep=0pt, leftmargin=*]
\item There is a sufficiently small $\eta>0$ for which Algorithm~\ref{algo:proximal-stochastic-gradient-descent} (using full-batch) is a descent algorithm with non-increasing objective values.
\item The sequence of hyperplane parameters $\B{\C{W}}$ generated from Algorithm~\ref{algo:proximal-stochastic-gradient-descent} is bounded.
\item The sequence of parameters $\B{\C{W}},\B{\C{O}}$ generated from Algorithm~\ref{algo:proximal-stochastic-gradient-descent} converges if $\eta>0$ is chosen as in Part~1.
\end{enumerate}
\end{theorem}
Theorem~\ref{conv-thm} shows that Algorithm~\ref{algo:proximal-stochastic-gradient-descent} is a convergent (descent) method for a suitably selected learning rate. The proof of Theorem~\ref{conv-thm} is presented in Supplement~\ref{app-proof}. Here, we note that Problem (\ref{eq:group-l0-l2}) is non-convex and non-smooth. Moreover, the activation function and therefore $\hat{\E}[L(\rvy, \B f( \rvx; \bm{\C {W}}, \bm{\C O}))]$ are not twice differentiable everywhere. These lead to technical challenges in proof, as discussed in the Supplement. 

\begin{figure}[!t]
\small
\centering
\begin{tabular}{cc}
    Lung & Madelon \\ 
    \includegraphics[width=0.22\textwidth]{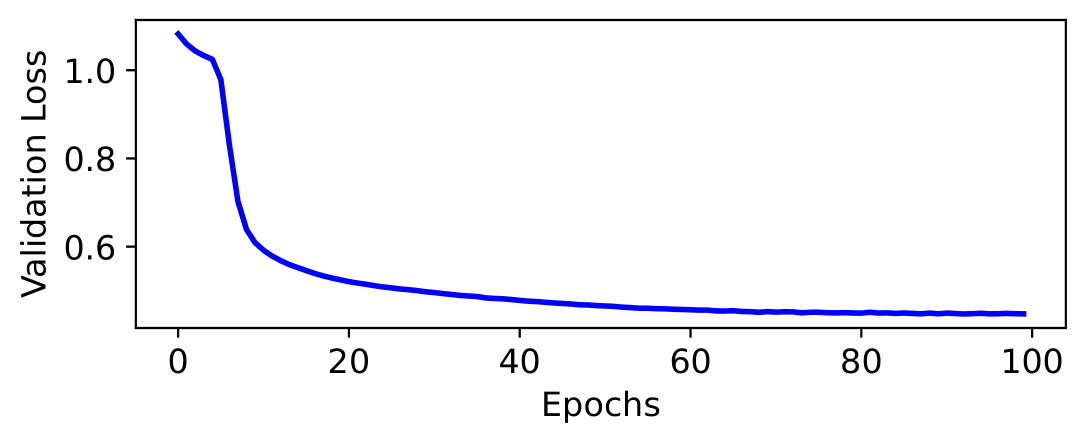}&  \includegraphics[width=0.22\textwidth]{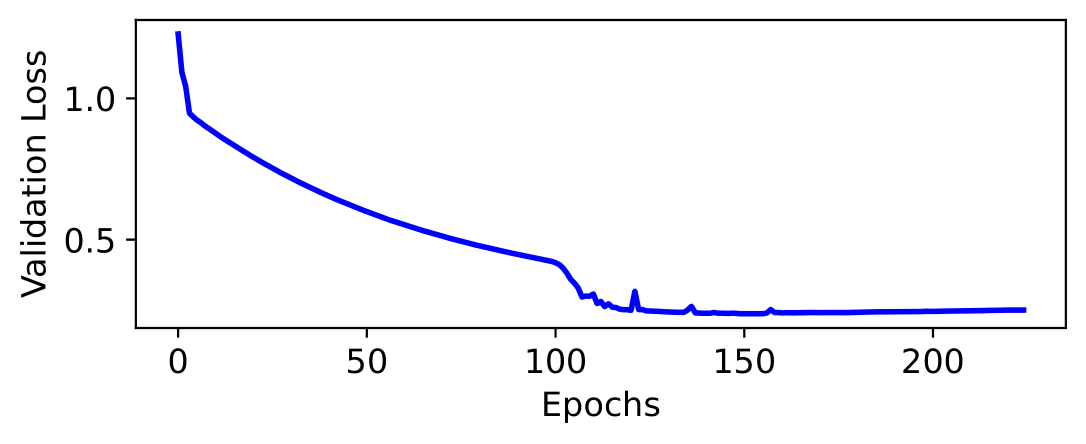}\\
    \includegraphics[width=0.22\textwidth]{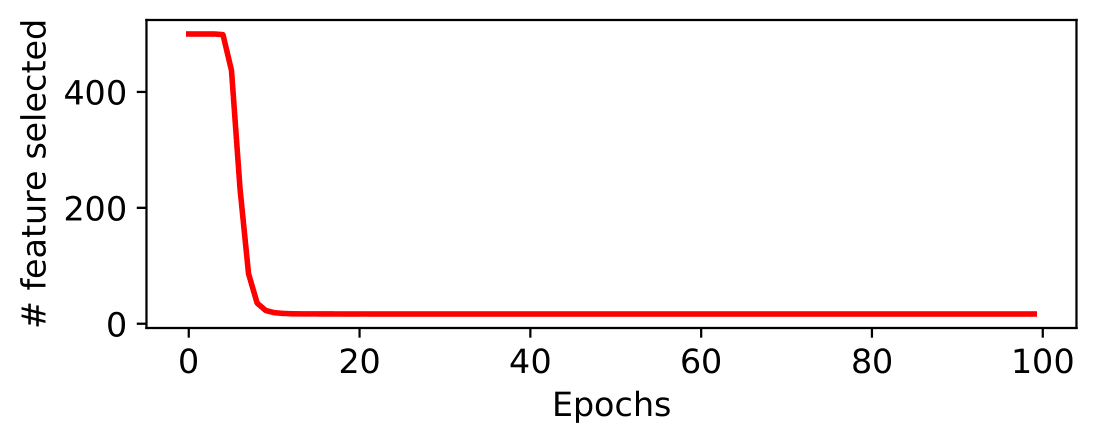}&  \includegraphics[width=0.22\textwidth]{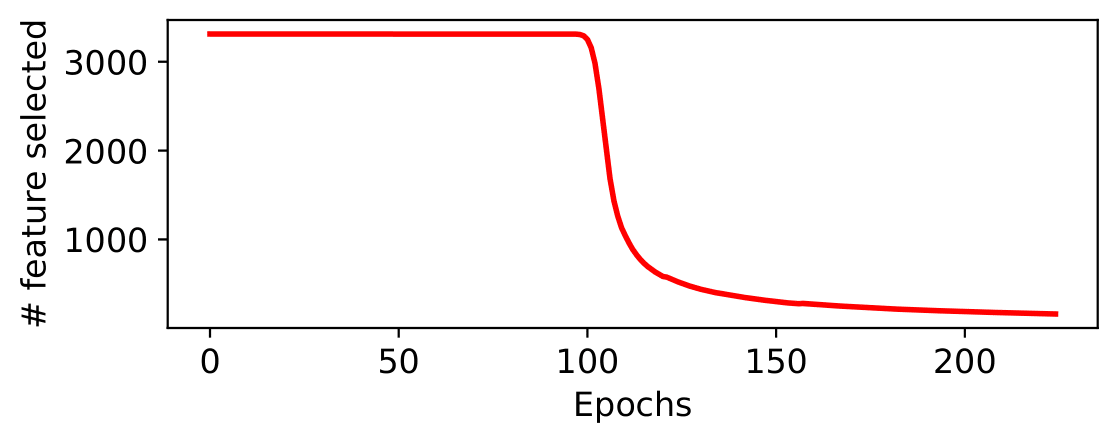}\\\end{tabular}
\caption{\emph{Trajectory of validation loss and feature sparsity during training with dense-to-sparse learning.}}  
\label{fig:dsl-trajectory}
\end{figure}

\begin{figure*}[!t]
\small
\centering
\begin{tabular}{c|c|c}
      Random Forests & XGBoost & \modelname~ (ours) \\ \hline 
      N=100 & N=100 & N=100 \\
      \includegraphics[width=0.31\textwidth]{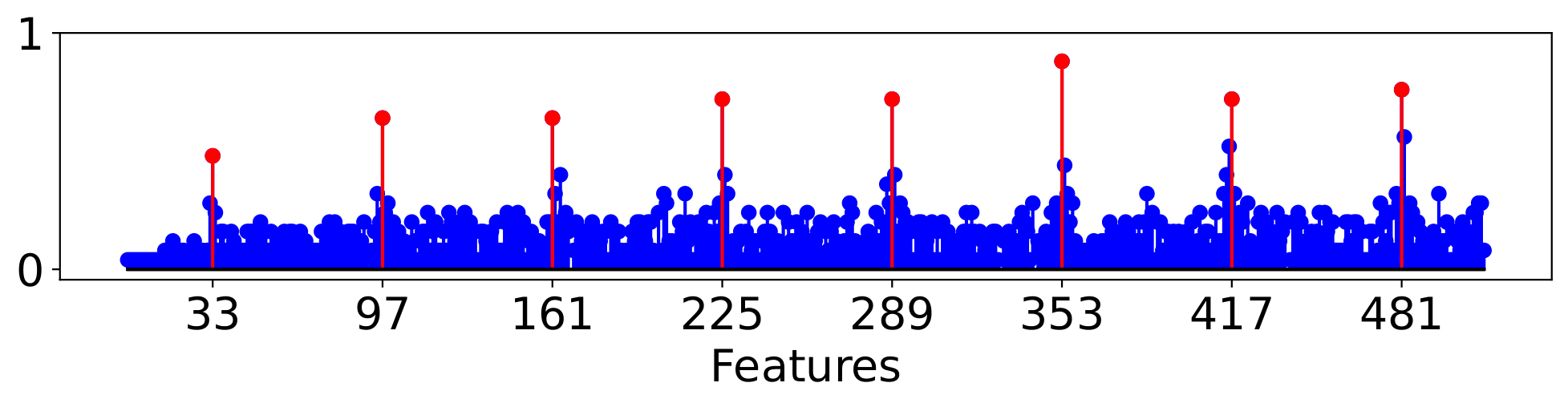}&  \includegraphics[width=0.31\textwidth]{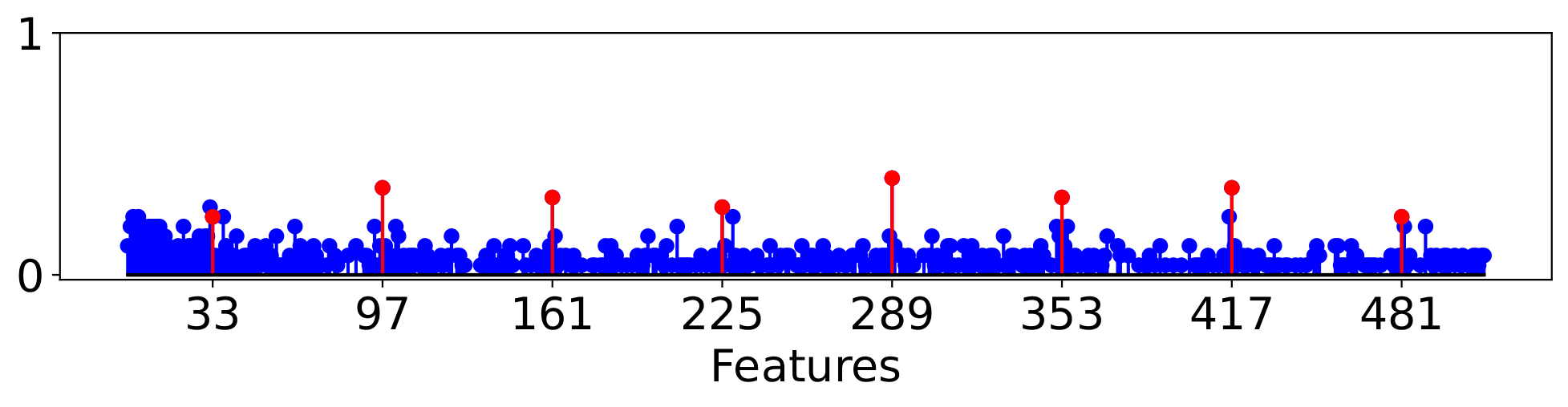}& \includegraphics[width=0.31\textwidth]{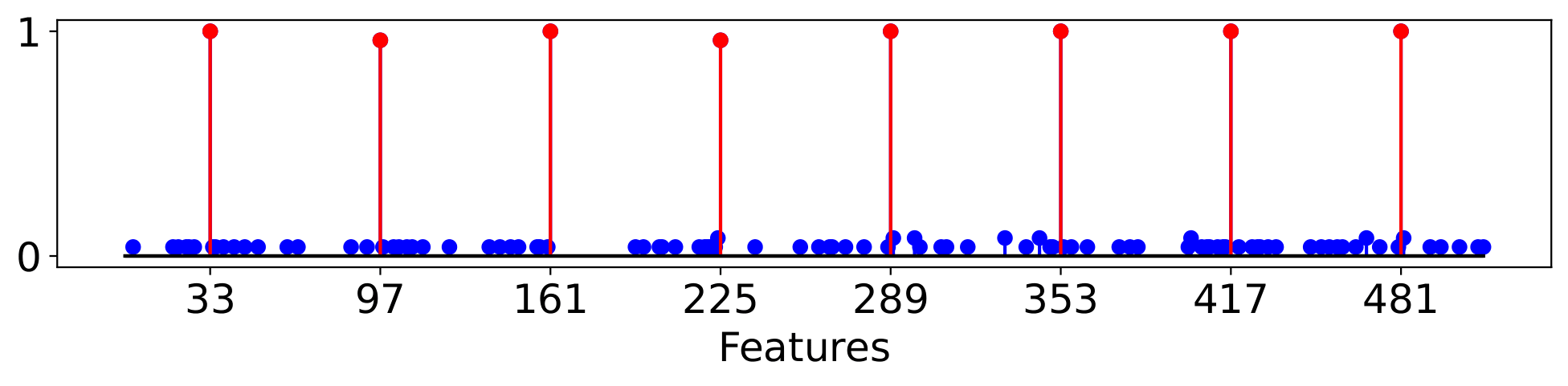} \\ \hline 
      N=200 & N=200 & N=200 \\
      \includegraphics[width=0.31\textwidth]{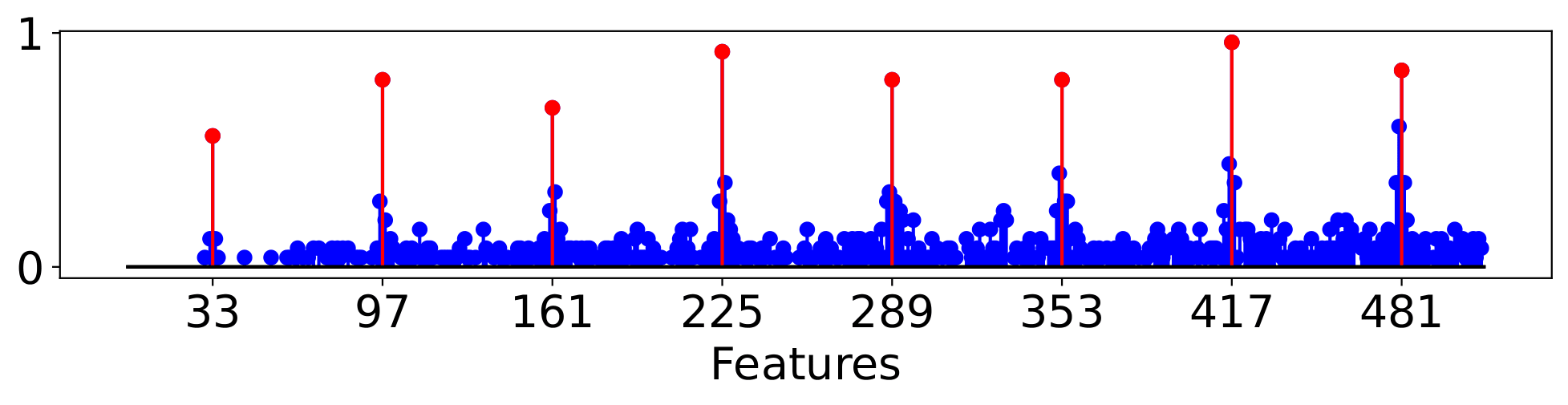}&  \includegraphics[width=0.31\textwidth]{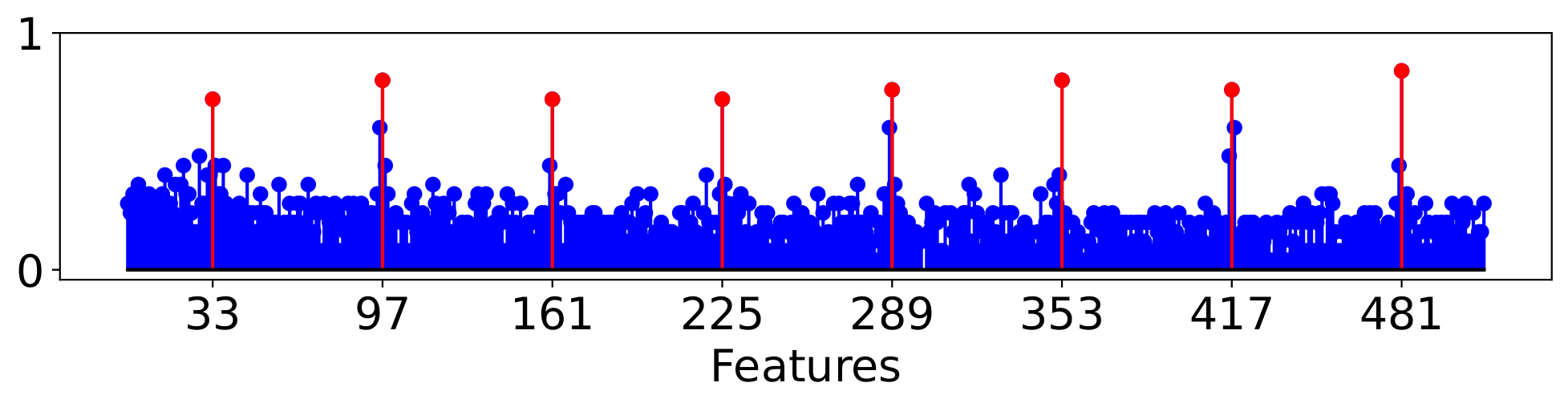}& \includegraphics[width=0.31\textwidth]{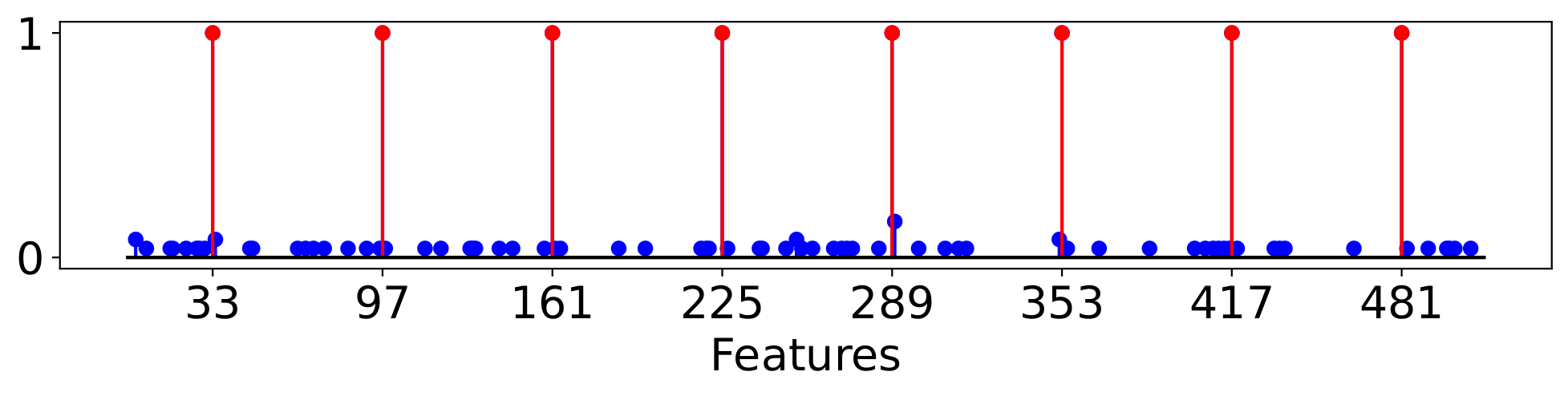} \\ \hline 
      N=1000 & N=1000 & N=1000 \\
      \includegraphics[width=0.31\textwidth]{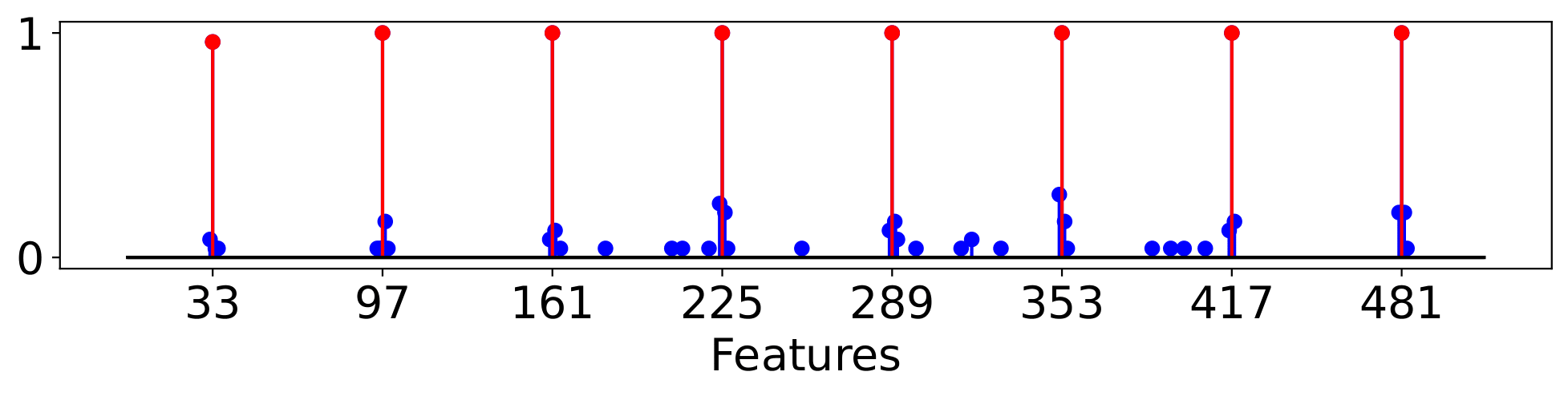}&  \includegraphics[width=0.31\textwidth]{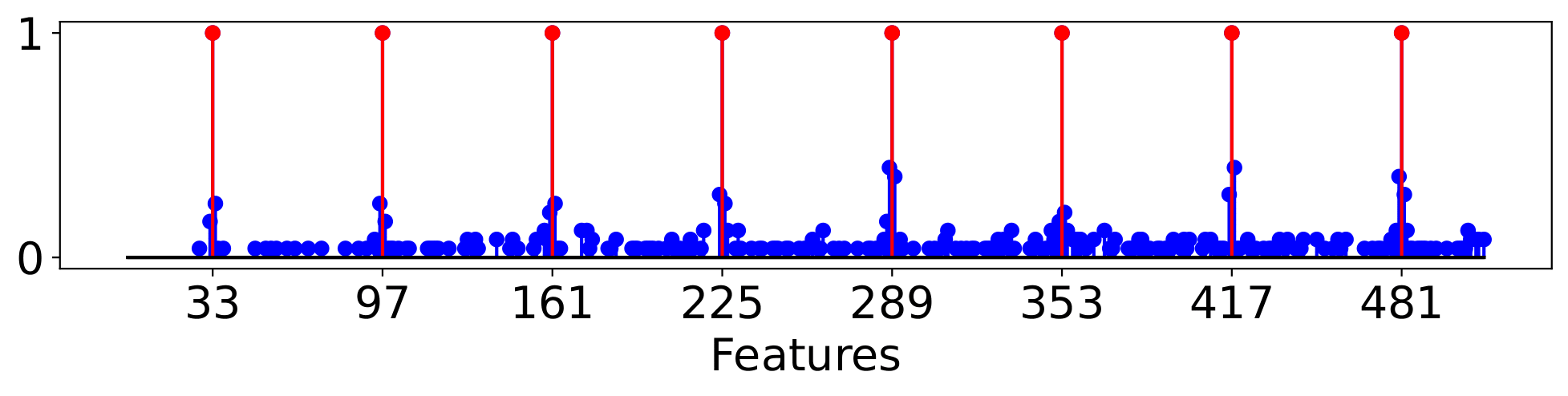}& \includegraphics[width=0.31\textwidth]{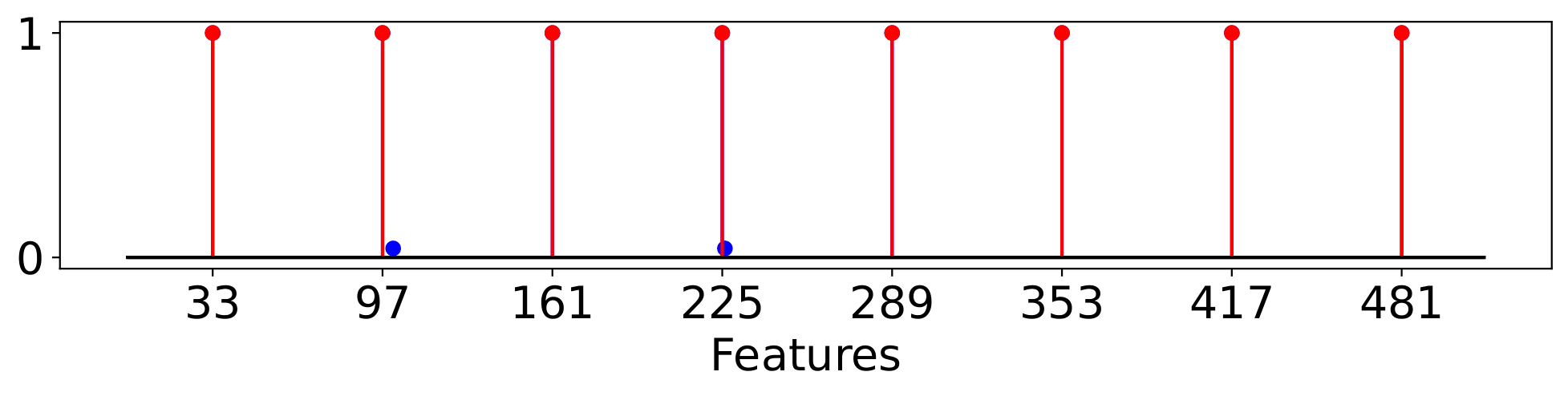} \\
\end{tabular}
\caption{\emph{Features selected by Random Forests, XGBoost and \modelname~ for different sample sizes}}      
\label{fig:rf-xgboost-ours}
\end{figure*}

\subsection{Dense-to-Sparse Learning (DSL)}
Prior work in feature selection  \citep{Lemhadri2021} recommend an interesting multi-stage approach: Train a dense model completely and then learn a progessively sparser model. At each sparsity level, the model is trained till convergence.
This approach appears to effectively leverage favorable generalization properties of the dense solution and preserves them after drifting into sparse local 
minima~\citep{Lemhadri2021}---in particular, this seems to perform better than sparsely training from scratch.
However, this approach can be expensive as it requires learning a dense model completely before starting the sparse-training process---the training runtime is likely to increase for higher sparsity settings (i.e., fewer number of features).

To reduce the computational cost of the above approach, we propose single-stage approach based on dense-to-sparse learning (DSL).
To this end, we anneal the sparsity-inducing penalty $\lambda_0$ from small to large values ($0 \rightarrow \gamma$) during the course of training.
We use an exponential annealing scheduler of the form: $\lambda_0 = \gamma(1-\text{exp}(-s*t))$, where $\gamma$ is the largest value of regularization penalty (corresponds to a fully sparse soft tree), $s$ controls the rate of increase of the regularization penalty and $t$ denotes the iteration step. 

We show the trajectory of the validation loss and the number of features selected during training with dense-to-sparse learning in Figure \ref{fig:dsl-trajectory}.
We empirically observed this scheduler to result in better out-of-sample accuracy and feature-sparsity tradeoffs (see Figure \ref{fig:dsl} in Sec. \ref{sec:results-dsl}).

\section{SYNTHETIC EXPERIMENTS}
\label{sec:simulation}
We first evaluate our proposed method using data with correlated features.
In real-world, high-dimensional datasets, features are often correlated.
Such correlations pose challenges for feature selection.
Existing tree ensemble toolkits, e.g., XGBoost and Random Forests, based on feature importance scores, may produce misleading results~\citep{Zhou2021}---any of the correlated features can work as a splitting variable, and the feature importance scores can get distributed (and hence deflated) among the correlated features.
Below, we consider a setting with correlated features and demonstrate the strong performance of \modelname~in terms of true support recovery on synthetic data.

\begin{table}[!t]
\caption{\emph{Test MSE, feature sparsity and support recovery metrics (F1-score) for a linear setting with correlated design matrix. \modelname~outperforms feature-importance-based methods across all metrics.}
}
\label{tab:synthetic-results}
\small
\resizebox{\columnwidth}{!}{\begin{tabular}{c|c|c|l|c|c|c}
$\sigma$ & $p$ & $N$ & Model & Test MSE & \#features & F1-score \\ \hline
\multirow{9}{*}{0.7} & \multirow{9}{*}{512} & \multirow{3}{*}{100} & RF & 6.49 $\pm$ 0.19 & 79 $\pm$ 17 & 0.21 $\pm$ 0.02 \\
 &  &  & XGBoost & 8.65 $\pm$ 0.27 & 32 $\pm$ 9 & 0.18 $\pm$ 0.03 \\
 &  &  & \modelname & \textbf{0.65} $\pm$ 0.12 & \textbf{12} $\pm$ 1 & \textbf{0.86} $\pm$ 0.04 \\ \cline{3-7} 
 &  & \multirow{3}{*}{200} & RF & 4.90 $\pm$ 0.15 & 40 $\pm$ 8 & 0.35 $\pm$ 0.03 \\
 &  &  & XGBoost & 5.97 $\pm$ 0.12 & 110 $\pm$ 21 & 0.18 $\pm$ 0.02 \\
 &  &  & \modelname & \textbf{0.34} $\pm$ 0.00 & \textbf{11} $\pm$ 1 & \textbf{0.89} $\pm$ 0.03 \\ \cline{3-7} 
 &  & \multirow{3}{*}{1000} & RF & 2.97 $\pm$ 0.03 & 11 $\pm$ 1 & 0.84 $\pm$ 0.02 \\
 &  &  & XGBoost & 1.81 $\pm$ 0.02 & 24 $\pm$ 1 & 0.50 $\pm$ 0.02 \\
 &  &  & \modelname & \textbf{0.26} $\pm$ 0.00 & \textbf{8} $\pm$ 0 & \textbf{1.00} $\pm$ 0.00 \\ \hline
\multirow{9}{*}{0.5} & \multirow{9}{*}{256} & \multirow{3}{*}{100} & RF & 6.24 $\pm$ 0.13 & 42 $\pm$ 11 & 0.35 $\pm$ 0.03 \\
 &  &  & XGBoost & 7.93 $\pm$ 0.27 & 35 $\pm$ 10 & 0.25 $\pm$ 0.02 \\
 &  &  & \modelname & \textbf{0.45} $\pm$ 0.06 & \textbf{10} $\pm$ 1 & \textbf{0.89} $\pm$ 0.03 \\ \cline{3-7} 
 &  & \multirow{3}{*}{200} & RF & 4.40 $\pm$ 0.13 & 18 $\pm$ 3 & 0.61 $\pm$ 0.04 \\
 &  &  & XGBoost & 5.61 $\pm$ 0.13 & 67 $\pm$ 12 & 0.26 $\pm$ 0.02 \\
 &  &  & \modelname & \textbf{0.31} $\pm$ 0.00 & \textbf{12} $\pm$ 1 & \textbf{0.87} $\pm$ 0.04 \\ \cline{3-7} 
 &  & \multirow{3}{*}{1000} & RF & 2.90 $\pm$ 0.02 & 9 $\pm$ 0 & 0.94 $\pm$ 0.01 \\
 &  &  & XGBoost & 1.45 $\pm$ 0.01 & 10 $\pm$ 0 & 0.91 $\pm$ 0.01 \\
 &  &  & \modelname & \textbf{0.26} $\pm$ 0.00 & \textbf{8} $\pm$ 0 & \textbf{1.00} $\pm$ 0.00 \\ \hline
\end{tabular}}
\end{table}
We evaluate our approach in a setting where the underlying data comes from a sparse linear model.
We generate the data matrix, $\B X\in\R^{N \times p}$ with samples drawn from a multivariate normal distribution $\C{N}(0, \B \Sigma)$ where entries of the covariance matrix $\B \Sigma$ are given by $\Sigma_{ij}=\sigma^{|i-j|}$.
We construct the response variable $\B y=\B X\B\beta^{*}+\epsilon$ where $\epsilon_i, i \in [N]$ are drawn independently from $\C{N}(0, 0.5)$. 
The locations of nonzero entries of $\B\beta^*$ are equi-spaced in $[p]$, with each nonzero entry one, and $\norm{\B\beta^{*}}_0=8$.
We experiment with a range of training set sizes $N \in \{100, 200, 1000\}$, correlation strengths $\sigma \in \{0.5, 0.7\}$, and number of total features $p \in \{256,512\}$. 
We evaluate the final performance averaged across 25 runs in terms of (i) test MSE, (ii) number of features selected and (iii) support recovery (computed via the F1-score between the true and recovered support).
More details are in Supplement Sec. \ref{supp-sec:simulation}.

\modelname~significantly outperforms both Random Forests and XGBoost in all three measures across various settings. With \modelname, we observe a 5-15 fold improvement in MSE performance and $9\%-65\%$ improvement in the support recovery metric (F1-score).
Table \ref{tab:synthetic-results} shows that even if the features are correlated, \modelname~successfully recovers the true support with high probability. We also visualize this in Figure \ref{fig:rf-xgboost-ours}. Indices corresponding to those in the true support are depicted in red.
This confirms the usefulness of our end-to-end feature selection approach.

\section{REAL DATA EXPERIMENTS}
\label{sec:experiments}
We study the performance of \modelname~on real-world datasets and compare against popular competing methods. We make the following comparisons:
(i) Single \modelsname~ vs other single tree baseline approaches with a limit on number of features,
(ii) \modelname~vs dense soft trees,
(iii) \modelname~vs wrapper-based feature selection tree toolkits,
(iv) \modelname~vs  neural network based embedded feature selection toolkits,
(v) Ablation study for dense-to-sparse learning for feature selection. 

\noindent\textbf{Implementation.}
\modelname~are implemented in TensorFlow Keras.     
Our code for \modelname~is available at\\
\url{https://github.com/mazumder-lab/SkinnyTrees}.

\noindent\textbf{Datasets.}
We use 14 open-source classification datasets (binary and multiclass) from various domains with a range of number of features: $20-100000$. Dataset details are in Table \ref{tab:classification-datasets} in Supplement.

\noindent\textbf{Tuning, Toolkits, and Details.} For all the experiments, we tune the hyperparameters using Optuna \citep{Akiba2019} with random search.
The number of selected features affects the AUC. Therefore, to treat all the methods in a fair manner, we tune the hyperparameter that controls the sparsity level using Optuna which optimizes the AUC across different $K$'s (budget on number of selected features) e.g., $0.25p$ or $0.50p$ on a held-out validation set.
Details are in the Supplement.

\subsection{Studying a single tree}
\label{sec:comparison-with-tao}
We first study feature selection for a single tree on 4 classification tasks. We study the performance of \modelsname~(a single soft tree with group $\ell_0-\ell_2$ regularization). 

\noindent\textbf{Competing Methods.}
We compare against:
\begin{enumerate}
    \item Decision tree with hyperplane splits (TAO~\citep{Carreira-Perpinan2018}) using $\ell_1$ regularization for node-level feature selection.
    \item Soft Tree with a Group Lasso~\citep{Yuan2006,Scardapane2017} regularization given by $\frac{\lambda_1}{\sqrt{m|\C{I}|}} {\sum}_{k \in [p]} \norm{\bm{\C{W}}_{k,:,:}}_2$.
\end{enumerate}

\begin{table}[!t]
\centering
\caption{\emph{Test AUC for TAO with $\ell_1$ regularization, single soft tree with Group Lasso and \modelsname ~(a single soft tree with Group $\ell_0-\ell_2$).}}
\label{tab:comparison-with-tao-l1}
\setlength{\tabcolsep}{5pt}
\resizebox{\columnwidth}{!}{\begin{tabular}{l|c|c|c}
 & Classical Tree & Soft Tree w/ & \modelsname \\ 
 & TAO & Group Lasso &  \\ \hline
Churn & 58.36 & 76.23 & \textbf{89.35}$\pm$0.15 \\
Satimage & 58.53 & 83.89 & \textbf{88.66}$\pm$0.05 \\
Texture & 58.90 & 93.83 & \textbf{98.42}$\pm$0.01 \\
Mice-protein & 57.13 & 87.88 & \textbf{99.19}$\pm$0.00 \\ \hline
\end{tabular}}
\end{table}

\begin{table}[!b]
\centering
\caption{Test AUC for \modelname~vs \emph{dense} Soft Trees. We also report feature compression.}
\label{tab:comparison-with-dense}
\setlength{\tabcolsep}{10pt}
\resizebox{\columnwidth}{!}{\begin{tabular}{l|cc|r}
 & Dense Trees & \modelname & Compression \\ \hline
Churn & \multicolumn{1}{l|}{91.15$\pm$0.09} & \multicolumn{1}{l|}{\textbf{93.20}$\pm$0.08} & 1.8$\times$ \\
Gisette & \multicolumn{1}{l|}{\textbf{99.81}$\pm$0.003} & \multicolumn{1}{l|}{\textbf{99.81}$\pm$0.002} & 1.5$\times$ \\
Arcene & \multicolumn{1}{l|}{89.57$\pm$0.11} & \multicolumn{1}{l|}{\textbf{90.80}$\pm$0.30} & 2$\times$ \\ 
Dorothea & \multicolumn{1}{l|}{90.67$\pm$0.03} & \multicolumn{1}{l|}{\textbf{92.15}$\pm$0.25} & 2.7$\times$ \\
Madelon & \multicolumn{1}{l|}{65.32$\pm$0.15} & \multicolumn{1}{l|}{\textbf{95.44}$\pm$0.05} & 26$\times$ \\ 
Smk & \multicolumn{1}{l|}{\textbf{84.10}$\pm$0.16} & \multicolumn{1}{l|}{79.29$\pm$0.22} & 253$\times$ \\
Cll & \multicolumn{1}{l|}{81.70$\pm$0.82} & \multicolumn{1}{l|}{\textbf{92.86}$\pm$0.31} & 189$\times$ \\ 
Gli & \multicolumn{1}{l|}{88.65$\pm$0.90} & \multicolumn{1}{l|}{\textbf{99.80}$\pm$0.07} & 619$\times$ \\
Lung & \multicolumn{1}{l|}{99.40$\pm$0.09} & \multicolumn{1}{l|}{\textbf{99.80}$\pm$0.03} & 253$\times$ \\
Tox & \multicolumn{1}{l|}{99.19$\pm$0.04} & \multicolumn{1}{l|}{\textbf{99.74}$\pm$0.02} & 189$\times$ \\ \hline
\end{tabular}}
\end{table}

\begin{table*}[!t]
\centering
\caption{\emph{Test AUC (\%) performance of \modelname~and feature-importance-based toolkits for \textbf{trees} for $25\%$ feature budget ($K=0.25p$). Bold and italics indicates best and runner-up models respectively.}}
\label{tab:sparse-soft-trees-vs-classical-trees}
\setlength{\tabcolsep}{15pt}
\resizebox{\textwidth}{!}{\begin{tabular}{l|l|c|c|c|c|c}
\multicolumn{1}{l|}{\multirow{1}{*}{Case}} & \multicolumn{1}{l|}{\multirow{1}{*}{Dataset}} & \multirow{1}{*}{Random Forests} & \multirow{1}{*}{XGBoost} & \multirow{1}{*}{LightGBM} & \multicolumn{1}{c|}{\multirow{1}{*}{CatBoost}} & \multicolumn{1}{c}{\modelname} \\ \hline
\multirow{7}{*}{$N<p$} & \multicolumn{1}{l|}{Lung} & \multicolumn{1}{l|}{93.80$\pm$0.28} & \multicolumn{1}{l|}{86.38$\pm$0.48} & \multicolumn{1}{l|}{80.83$\pm$1.87} & \multicolumn{1}{l|}{\textit{94.72}$\pm$0.56} & \multicolumn{1}{l}{\textbf{99.80}$\pm$0.03} \\ 
& \multicolumn{1}{l|}{Tox} & \multicolumn{1}{l|}{94.52$\pm$0.14} & \multicolumn{1}{l|}{\textit{97}.10$\pm$0.09} & \multicolumn{1}{l|}{95.94$\pm$0.54} & \multicolumn{1}{l|}{95.95$\pm$0.14} & \multicolumn{1}{l}{\textbf{99.74}$\pm$0.02} \\ 
& \multicolumn{1}{l|}{Arcene}  & \multicolumn{1}{l|}{74.80$\pm$0.36} & \multicolumn{1}{l|}{76.36$\pm$0.16}  & \multicolumn{1}{l|}{\textit{76.92}$\pm$0.36} & \multicolumn{1}{l|}{76.64$\pm$0.22} & \multicolumn{1}{l}{\textbf{80.80}$\pm$0.30} \\ 
& \multicolumn{1}{l|}{Cll} & \multicolumn{1}{l|}{94.08$\pm$0.27} & \multicolumn{1}{l|}{\textit{94.21}$\pm$0.18} & \multicolumn{1}{l|}{55.17$\pm$1.14} & \multicolumn{1}{l|}{\textbf{94.41}$\pm$0.26} & \multicolumn{1}{l}{92.86$\pm$0.31} \\ 
& \multicolumn{1}{l|}{Smk} & \multicolumn{1}{l|}{77.78$\pm$0.20} & \multicolumn{1}{l|}{76.88$\pm$0.40} & \multicolumn{1}{l|}{67.29$\pm$0.91} & \multicolumn{1}{l|}{\textit{78.44}$\pm$0.41} & \multicolumn{1}{l}{\textbf{79.29}$\pm$0.22} \\ 
& \multicolumn{1}{l|}{Gli} & \multicolumn{1}{l|}{87.35$\pm$1.08} & \multicolumn{1}{l|}{82.37$\pm$1.47} & \multicolumn{1}{l|}{71.28$\pm$2.05} & \multicolumn{1}{l|}{\textit{91.31}$\pm$0.73} & \multicolumn{1}{l}{\textbf{99.80}$\pm$0.07} \\ 
& \multicolumn{1}{l|}{Dorothea}  & \multicolumn{1}{l|}{\textit{89.71}$\pm$0.12} & \multicolumn{1}{l|}{89.09$\pm$0.09} & \multicolumn{1}{l|}{88.14$\pm$0.18} & \multicolumn{1}{l|}{88.50$\pm$0.27} & \multicolumn{1}{l}{\textbf{90.87}$\pm$0.02} \\ 
\hline
\multirow{7}{*}{$N>p$} & \multicolumn{1}{l|}{Churn}  & \multicolumn{1}{l|}{83.79$\pm$0.24} & \multicolumn{1}{l|}{\textit{88.68}$\pm$0.06} & \multicolumn{1}{l|}{86.33$\pm$0.08} & \multicolumn{1}{l|}{83.73$\pm$0.06} & \multicolumn{1}{l}{\textbf{91.38}$\pm$0.08}  \\ 
& \multicolumn{1}{l|}{Satimage} & \multicolumn{1}{l|}{97.62$\pm$0.005} & \multicolumn{1}{l|}{\textbf{98.23}$\pm$0.01} & \multicolumn{1}{l|}{94.00$\pm$0.05} & \multicolumn{1}{l|}{95.11$\pm$0.05} & \multicolumn{1}{l}{\textit{98.05}$\pm$0.01}  \\ 
& \multicolumn{1}{l|}{Texture}  & \multicolumn{1}{l|}{99.60$\pm$0.003} & \multicolumn{1}{l|}{\textit{99.94}$\pm$0.001} & \multicolumn{1}{l|}{96.14$\pm$0.03} & \multicolumn{1}{l|}{94.90$\pm$0.07} & \multicolumn{1}{l}{\textbf{99.97}$\pm$0.002}  \\ 
& \multicolumn{1}{l|}{Mice-protein}  & \multicolumn{1}{l|}{99.30$\pm$0.01} & \multicolumn{1}{l|}{\textbf{99.77}$\pm$0.01} & \multicolumn{1}{l|}{89.59$\pm$0.22} & \multicolumn{1}{l|}{95.03$\pm$0.07} &  \multicolumn{1}{l}{\textit{99.59}$\pm$0.02} \\ 
& \multicolumn{1}{l|}{Isolet}  & \multicolumn{1}{l|}{99.17$\pm$0.002} & \multicolumn{1}{l|}{99.86$\pm$0.002} & \multicolumn{1}{l|}{97.62$\pm$0.003} & \multicolumn{1}{l|}{\textit{99.89}$\pm$0.001} & \multicolumn{1}{l}{\textbf{99.94}$\pm$0.01} \\ 
& \multicolumn{1}{l|}{Madelon}  & \multicolumn{1}{l|}{94.11$\pm$0.02} & \multicolumn{1}{l|}{\textit{94.65}$\pm$0.01}  & \multicolumn{1}{l|}{86.46$\pm$0.08} & \multicolumn{1}{l|}{\textbf{96.41}$\pm$0.01} & \multicolumn{1}{l}{94.14$\pm$0.09} \\ 
& \multicolumn{1}{l|}{Gisette}  & \multicolumn{1}{l|}{98.99$\pm$0.004} & \multicolumn{1}{l|}{\textit{99.64}$\pm$0.004} & \multicolumn{1}{l|}{98.09$\pm$0.50} & \multicolumn{1}{l|}{99.57$\pm$0.01} & \multicolumn{1}{l}{\textbf{99.81}$\pm$0.002} \\ 
\hline
& \multicolumn{1}{l|}{Average} & \multicolumn{1}{l|}{91.75} & \multicolumn{1}{l|}{91.65} & \multicolumn{1}{l|}{84.56} & \multicolumn{1}{l|}{91.76}  & \multicolumn{1}{l}{\textbf{94.72}}
\end{tabular}}
\end{table*}
\noindent\textbf{Results.} The numbers for classical-tree based TAO with $\ell_1$ regularization and soft tree with Group Lasso regularization are shown in Table \ref{tab:comparison-with-tao-l1} for 50\% sparsity budget. Results for \modelsname~are also shown. We see a huge gain in test AUC performance across all 4 datasets with \modelsname~in comparison with TAO and group lasso variant of a soft tree.
This confirms that in the context of feature selection at the ensemble level, a node-level $\ell_1$ penalty is sub-optimal. Similarly, it also suggests that joint selection and shrinkage using Group Lasso can be less useful than Group $\ell_0-\ell_2$.

\subsection{\modelname~vs Dense Soft Trees}
\label{sec:results-dense-vs-sparse}
In this section, we compare our sparse trees with dense soft trees. For dense soft trees, we use FASTEL \citep{Ibrahim2022} (an efficient state-of-the-art toolkit for training soft tree ensembles). 
We present test AUC performances in Table \ref{tab:comparison-with-dense}. 
\modelname~matches or outperforms dense soft trees in 10 datasets.
Notably, we observe a $30\%$ gain in test AUC on Madelon dataset with \modelname. 
We also observe $11\%$ improvements in test AUC on Cll and Gli datasets.
Additionally, sparse trees achieve $1.3\!\times\!-620\times$ feature compression on 10 datasets.
Note that in soft trees, feature compression has a direct impact on model compression---this has reduced storage requirements and results in faster inference.
We observed up to $10\times$ faster inference times for \modelname~compared to dense soft trees for compression rates of $1.5\!\times\!-620\times$.

\subsection{\modelname~vs Classical Trees}
\label{sec:results-soft-vs-classical}
We compare \modelname~  against wrapper methods for feature selection as available from ensembles of classical trees (e.g., Random Forests,  XGBoost, LightGBM, and CatBoost) on real-world datasets.
For \modelname, we use the combined dense-to-sparse scheduler. The tuning protocol and hyperparameters for all methods are reported in the Supplement Sec. \ref{supp-sec:soft-vs-classical}.
The results are in Table \ref{tab:sparse-soft-trees-vs-classical-trees}.
\modelname~leads on 10 datasets. 
In contrast, other methods lead on 2 datasets.
In terms of test AUC, \modelname~outperforms LightGBM by $10.2\%$ (up to $37.7\%$), XGBoost by $3.1\%$ (upto $17.4\%$), Random Forests by $3\%$ (up to $12.5\%$) and CatBoost by $3\%$ (up to $8.5\%$).
Overall, \modelname~provides a strong alternative to existing wrapper-based methods.

Additional comparison with ControlBurn \citep{Liu2021} is included in Supplement Sec. \ref{supp-sec:controlburn}. 
\modelname~also outperforms ControlBurn, achieving $2\%$ (up to $6\%$) improvement in AUC.

\begin{table*}[!t]
\centering
\caption{\emph{Test AUC (\%) performance of \modelname~and embedded feature selection methods from \textbf{neural networks} (LassoNet, AlgNet, DFS) for $25\%$ feature budget.}}
\label{tab:sparse-soft-trees-vs-nn}
\setlength{\tabcolsep}{20pt}
\resizebox{\textwidth}{!}{\begin{tabular}{l|c|c|c|c|c}
\multicolumn{1}{l|}{\multirow{1}{*}{Case}} & \multicolumn{1}{l|}{\multirow{1}{*}{Dataset}} & \multirow{1}{*}{LassoNet} & \multirow{1}{*}{AlgNet} & \multirow{1}{*}{DFS} & \multicolumn{1}{c}{\modelname} \\ \hline
\multirow{7}{*}{$N<p$} & \multicolumn{1}{l|}{Lung}  & \multicolumn{1}{l|}{99.56$\pm$0.02} & \multicolumn{1}{l|}{56.72$\pm$1.34} & \multicolumn{1}{l|}{\textit{98.05}$\pm$0.36} & \multicolumn{1}{l|}{\textbf{99.80}$\pm$0.03} \\ 
& \multicolumn{1}{l|}{Tox}  & \multicolumn{1}{l|}{99.63$\pm$0.03} & \multicolumn{1}{l|}{51.01$\pm$0.70} & \multicolumn{1}{l|}{\textit{99.13}$\pm$0.24} & \multicolumn{1}{l|}{\textbf{99.74}$\pm$0.02} \\ 
& \multicolumn{1}{l|}{Arcene}  & \multicolumn{1}{l|}{66.26$\pm$0.29} & \multicolumn{1}{l|}{51.00$\pm$1.77} & \multicolumn{1}{l|}{69.37$\pm$0.59} & \multicolumn{1}{l|}{\textbf{80.80}$\pm$0.30} \\ 
& \multicolumn{1}{l|}{Cll}  & \multicolumn{1}{l|}{\textbf{95.04}$\pm$0.24} & \multicolumn{1}{l|}{64.96$\pm$2.54} & \multicolumn{1}{l|}{92.85$\pm$0.32} & \multicolumn{1}{l|}{\textit{92.86}$\pm$0.31} \\ 
& \multicolumn{1}{l|}{Smk}  & \multicolumn{1}{l|}{\textbf{85.56}$\pm$0.31} & \multicolumn{1}{l|}{53.92$\pm$2.16} & \multicolumn{1}{l|}{\textit{79.62}$\pm$0.29} & \multicolumn{1}{l|}{79.29$\pm$0.22} \\ 
& \multicolumn{1}{l|}{Gli}  & \multicolumn{1}{l|}{\textit{97.78}$\pm$0.56} & \multicolumn{1}{l|}{61.37$\pm$4.27} & \multicolumn{1}{l|}{92.25$\pm$0.64} & \multicolumn{1}{l|}{\textbf{99.80}$\pm$0.07} \\ 
& \multicolumn{1}{l|}{Dorothea}  & \multicolumn{1}{l|}{out of mem.} & \multicolumn{1}{l|}{81.74$\pm$0.91} & \multicolumn{1}{l|}{\textit{85.18*}} & \multicolumn{1}{l|}{\textbf{90.87}$\pm$0.02} \\ 
\hline
\multirow{7}{*}{$N>p$} & \multicolumn{1}{l|}{Churn}  & \multicolumn{1}{l|}{67.34$\pm$1.58} & \multicolumn{1}{l|}{70.10$\pm$0.91} & \multicolumn{1}{l|}{\textit{85.70}$\pm$0.52} & \multicolumn{1}{l|}{\textbf{91.38}$\pm$0.08} \\ 
& \multicolumn{1}{l|}{Satimage} & \multicolumn{1}{l|}{94.73$\pm$0.19} & \multicolumn{1}{l|}{95.30$\pm$0.20} & \multicolumn{1}{l|}{\textit{97.39}$\pm$0.04} & \multicolumn{1}{l|}{\textbf{98.05}$\pm$0.01} \\ 
& \multicolumn{1}{l|}{Texture}  & \multicolumn{1}{l|}{98.02$\pm$0.40} & \multicolumn{1}{l|}{76.24$\pm$1.94} & \multicolumn{1}{l|}{\textit{99.63}$\pm$0.04} & \multicolumn{1}{l|}{\textbf{99.97}$\pm$0.002} \\ 
& \multicolumn{1}{l|}{Mice-protein}  & \multicolumn{1}{l|}{94.90$\pm$0.26} & \multicolumn{1}{l|}{89.07$\pm$0.59} & \multicolumn{1}{l|}{\textit{99.04}$\pm$0.03} &  \multicolumn{1}{l|}{\textbf{99.59}$\pm$0.02} \\ 
& \multicolumn{1}{l|}{Isolet}  & \multicolumn{1}{l|}{99.64$\pm$0.01} & \multicolumn{1}{l|}{70.21$\pm$2.92} & \multicolumn{1}{l|}{\textit{99.92}$\pm$0.00} & \multicolumn{1}{l|}{\textbf{99.94}$\pm$0.01} \\ 
& \multicolumn{1}{l|}{Madelon}  & \multicolumn{1}{l|}{81.15$\pm$2.53} & \multicolumn{1}{l|}{68.55$\pm$1.42} & \multicolumn{1}{l|}{\textit{92.73}$\pm$0.45} & \multicolumn{1}{l|}{\textbf{94.14}$\pm$0.09} \\ 
& \multicolumn{1}{l|}{Gisette}  & \multicolumn{1}{l|}{99.81$\pm$0.002} & \multicolumn{1}{l|}{73.49$\pm$1.54} & \multicolumn{1}{l|}{$\textit{99.72}^{*}$} & \multicolumn{1}{l|}{\textbf{99.81}$\pm$0.002} \\ 
\hline
& \multicolumn{1}{l|}{Average}  & \multicolumn{1}{l|}{90.43**} & \multicolumn{1}{l|}{68.83} & \multicolumn{1}{l|}{92.18} & \multicolumn{1}{l|}{\textbf{94.72}} \\ 
\hline
\multicolumn{6}{l}{${}^{*}$DFS is very time-consuming to run, we report the test AUC for best trial (based on validation AUC)} \\
\multicolumn{6}{l}{during tuning on Gisette and Dorothea.} \\
\multicolumn{6}{l}{${}^{**}$Adjusted Average: $\frac{90.72}{(94.72*14-90.87)/13}*94.72=90.43$.} \\
\end{tabular}}
\end{table*}

\subsection{\modelname~vs Neural Networks}
\label{sec:results-soft-vs-nn}
In this paper, we pursue embedded feature selection methods for \emph{tree ensembles}.
However, for completeness, we compare \modelname~ against some state-of-the-art embedded feature selection methods from neural networks, namely LassoNet \citep{Lemhadri2021}, AlgNet \citep{Vu2020} and DFS \citep{Chen2021}. Details are in Supplement Sec. \ref{supp-sec:soft-vs-neural-networks}.

\noindent\textbf{Results.} We report AUC performance for $25\%$ feature budget in Table \ref{tab:sparse-soft-trees-vs-nn}.
\modelname~leads across many datasets.
In terms of test AUC, \modelname~outperforms LassoNet by $4.3\%$ (up to $24\%$), AlgNet by $25.9\%$ (up to $49\%$), and DFS by $2.6\%$ (up to $11.4\%$).

\subsection{Dense-to-Sparse Learning}
\label{sec:results-dsl}

\begin{figure}[!b]
\small
\centering
\begin{tabular}{c}
    Smk \\ 
    \includegraphics[width=0.475\textwidth]{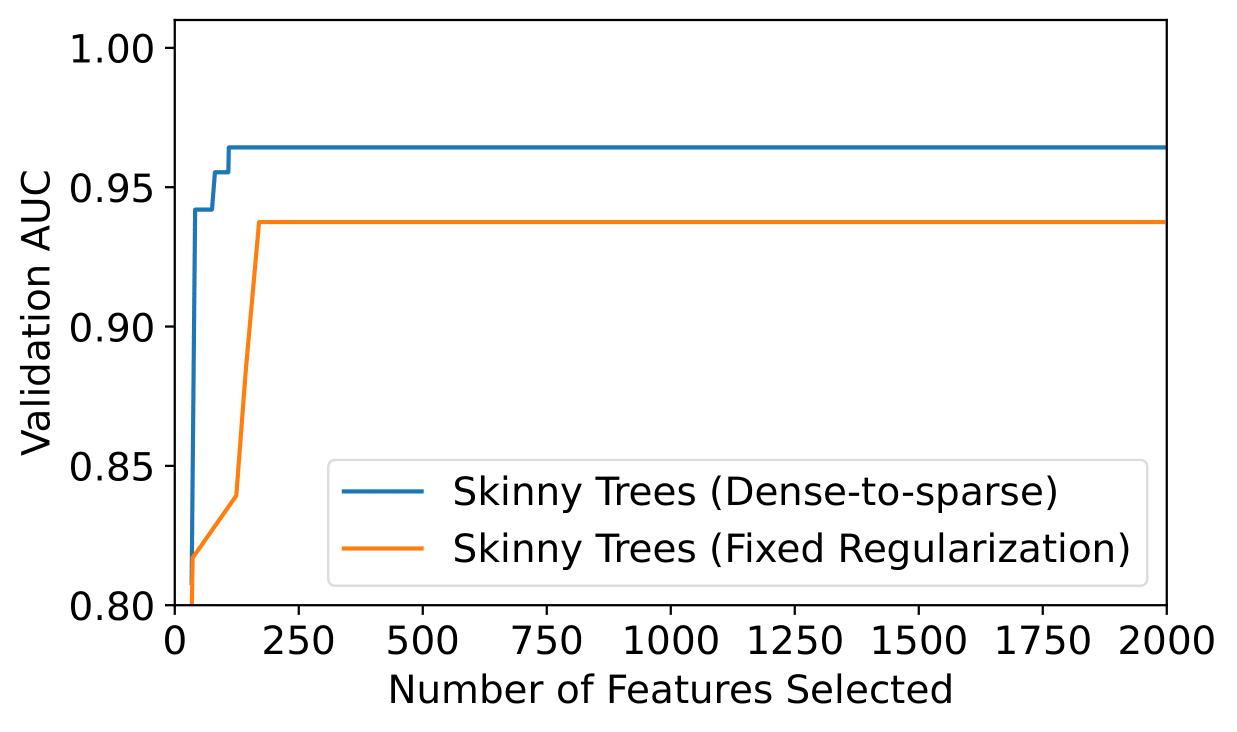} \\ 
    Dorothea \\ 
    \includegraphics[width=0.475\textwidth]{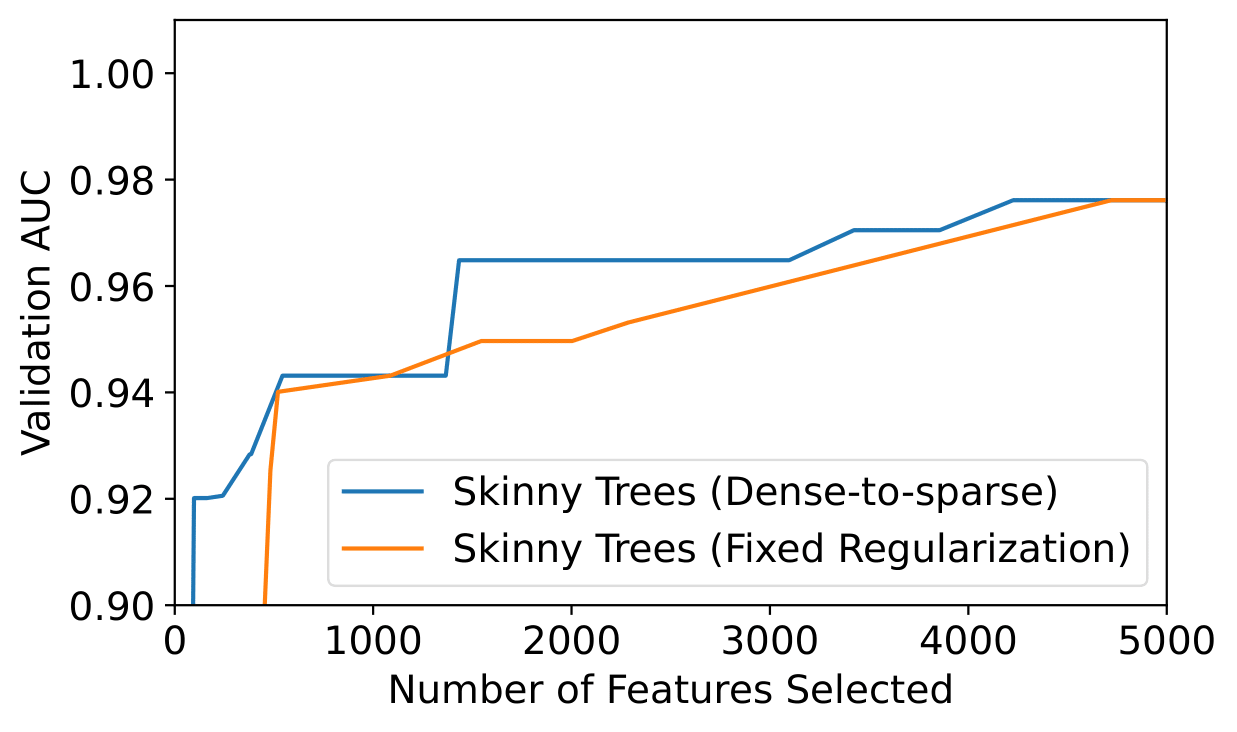}\\
\end{tabular}
\caption{Performance without/with Dense-to-sparse learning for different feature selection budgets.}  
\label{fig:dsl}
\end{figure}
\looseness=-1 We perform an ablation study in which we compare the predictive performance achieved with dense-to-sparse learning (DSL) over a range of feature selection budgets.
Tuning details are in the Supplement Sec. \ref{supp-sec:schedulers}. 
The results are reported in Figure \ref{fig:dsl}. 
Interestingly, we improve in test AUC across a range of feature selection budgets with dense-to-sparse learning over fixed regularization tuning.

\subsection{Discussion on training times.} \modelname~is very competitive in terms of training times in comparison to existing toolkits. We compared timings on a single Tesla V100 GPU.  
For example, on dorothea dataset, \modelname~trained in under 3 minutes for optimal hyperparameter setting. XGBoost took 10 minutes. In contrast, DFS took 45 hours.

\section{CONCLUSION}
We introduce an end-to-end optimization approach for joint feature selection and tree ensemble learning.
Our approach is based on differentiable trees with group $\ell_0-\ell_2$ regularization.
We use a simple but effective proximal mini-batch gradient descent algorithm and present convergence guarantees.
We propose a dense-to-sparse regularization scheduling approach that can lead to better feature-sparsity-vs-accuracy tradeoffs. 
We demonstrate on various datasets that our toolkit \modelname~can improve feature selection over several state-of-the-art wrapper-based feature selection methods in trees and embedded feature selection methods in neural networks. 

\section{Acknowledgments}
This research was supported in part, by grants from the Office of Naval Research (N000142112841), and Liberty Mutual Insurance. 
The authors acknowledge the MIT SuperCloud \citep{reuther2018interactive} and Lincoln Laboratory for providing HPC resources that have contributed to the research reported within this paper.

\bibliography{references}

\appendix
\onecolumn
\section*{SUPPLEMENTARY MATERIAL}

\setcounter{table}{0}
\renewcommand{\thetable}{S\arabic{table}}%
\setcounter{figure}{0}
\renewcommand{\thefigure}{S\arabic{figure}}%
\setcounter{equation}{0}
\renewcommand{\theequation}{S\arabic{equation}}
\setcounter{section}{0}
\renewcommand{\thesection}{S\arabic{section}}%
\setcounter{footnote}{0}

\section{NOTATIONS AND ACRONYMS}
\noindent\textbf{Notation} Table \ref{tab:notation} lists the notation used throughout the paper.
\begin{table}[!h]
    \caption{List of notation used.}
    \label{tab:notation}
    \centering
    \resizebox{0.9\textwidth}{!}{\begin{tabular}{|c|c|p{125mm}|}
    \hline
    \textbf{Notation}     & \textbf{Space or Type} & \textbf{Explanation}  \\
    \hline
    $[n]$ & Set & The set of integers $\{1,2,....,n\}$. \\
    \hline
    $\B 1_m$ & $\R^m$ & Vector with all coordinates equal to 1. \\
    \hline
    $\B U$ & $\R^{m,n}$ & Matrix with elements $((U_{ij}))$ \\
    \hline
    $\B u \cdot \B v$ & $\R$ & A dot product between two vectors $\B u, \B v$. \\ 
    \hline
    $\B U \cdot \B v$ & $\R^{n}$ & A dot product between a matrix $\B U \in R^{m,n}$ and a vector $\B v \in \R^m$ is denoted as $\B U \cdot \B v = \B U^T v \in \R^{n}$. \\
    \hline
    $\mathcal{X}$ &$\mathbb{R}^{p}$ & Input feature space.\\
    \hline
    $\mathcal{Y}$ & $\mathbb{R}^{c}$ & Output (label) space. \\
    \hline
    $m$ & $\mathbb{Z}_{> 0}$ & Number of trees in \modelname. \\
    \hline
    $\B f(\B x)$ & Function & The output of \modelname, a function that takes an input sample and returns a logit which corresponds to the sum of all the trees in the ensemble. Formally, $\B f: \mathcal{X} \to \mathbb{R}^{c}.$ \\
    \hline
    $\B f^j(\B x)$ & Function & A single perfect binary tree which takes an input sample and returns a logit, i.e., $\B f^j: \mathcal{X} \to \mathbb{R}^{c}$.\\
    \hline
    $d$ & $\mathbb{Z}_{> 0}$ & The depth of tree $\B f^j$. \\
    \hline
    $\mathcal{I}^j$ & Set & The set of internal (split) nodes in $\B f^j$. \\
    \hline
    $\mathcal{I}$ & Set & The set of internal (split) supernodes in $\B f$. \\
    \hline
    $\mathcal{L}^j$ & Set & The set of leaf nodes in $\B f^j$. \\
    \hline
    $\mathcal{A}(i)$ & Set & The set of ancestors of node $i$.\\ 
    \hline
    $\{ x \to i \}$ & Event & The event that sample  $x \in \mathbb{R}^{p}$ reaches node $i$. \\
    \hline
    $\B w_i$ & $\mathbb{R}^{p}$ & Weight vector of internal node $i$ (trainable). Defines the hyperplane split used in sample routing. \\
    \hline
    $\B W_i$ & $\mathbb{R}^{p,m}$ & Matrix of all weights in the internal supernode $i$ of the ensemble in the tensor-formulation. \\
    \hline
    $\bm{\C{W}}$ & $\mathbb{R}^{p,m,|\mathcal{I}|}$ & Tensor of all weights across all internal supernodes in the ensemble. \\
    \hline
    $\bm{\C{W}}_{k,:,:}$ & $\mathbb{R}^{m ,|\mathcal{I}|}$ & Matrix of all weights for $k$-th feature/covariate across all internal supernodes in the ensemble. \\
    \hline
    $S$ & Function & Activation function $\mathbb{R} \to [0,1]$ \\
    \hline 
    $S(\B w_i \cdot \B x)$ & $[0,1]$ & Probability (proportion) that internal node $i$ routes $x$ to the left. \\
    \hline
    $S'(v)$ & Function & The derivative of $S(v)$ \\
    \hline
    $[l \swarrow i]$ & Event & The event that leaf $l$ belongs to the left subtree of node $i \in \mathcal{I}$. \\
    \hline
    $[l \searrow i]$ & Event & The event that leaf $l$ belongs to the right subtree of node $i \in \mathcal{I}$. \\
    \hline
     $\B o_{l}$ & $\mathbb{R}^{c}$ & Leaf $l$'s weight vector (trainable). \\
    \hline
    $\B O_l$ & $\mathbb{R}^{m,c}$ & Matrix of weights in superleaf $l$. \\
    \hline
    $\bm{\C{O}}$ & $\mathbb{R}^{ m,c,|\mathcal{L}|}$ & Tensor of weights across the superleaves in the ensemble. \\
    \hline
    $L$ & Function & Loss function for training (e.g., cross-entropy). \\ 
    \hline
    $\C{Q}$ & Set & Unknown (learnable) subset of features of size at most $K$. \\
    \hline
    $z_k$ & $\{0,1\}$ & Binary variable controls whether $k$-th feature is on or off in Problem (2). \\
    \hline
    $\B z$ & $\{0,1\}^p$ & Binary vector controlling which features are on or off in Problem (2). \\
    \hline
    $\lambda_0$ & ${\mathbb{R}}_{\geq 0}$  & Non-negative $\ell_0$ regularization parameter controlling the number of features selected in Problem (3) \\
    \hline
    $\lambda_2$ & ${\mathbb{R}}_{\geq 0}$  & Non-negative $\ell_2$ regularization parameter controlling the shrinkage in Problem (3) \\
    \hline
    $\lambda_1$ & ${\mathbb{R}}_{\geq 0}$  & Non-negative $\ell_0$ regularization parameter controlling the number of features selected and shrinkage in Problem (\ref{eq:group-lasso}) \\
    \hline
    $\gamma$ & ${\mathbb{R}}_{\geq 0}$ & Non-negative scaling parameter for the exponential ramp-up of $\ell_0$-penalty in dense-to-sparse learning. \\
    \hline
    $s$ & ${\mathbb{R}}_{\geq 0}$ & Non-negative temperature parameter for controlling the ramping rate of the $\ell_0$-penalty in dense-to-sparse learning. \\
    \hline
    $\eta$ & ${\mathbb{R}}_{\geq 0}$ & Learning rate parameter for proximal mini-batch gradient descent. \\
    \hline
    \end{tabular}}
\end{table}

\clearpage
\noindent\textbf{Acronyms} Table \ref{tab:acronyms} lists the acronyms used throughout the paper.
\begin{table}[!h]
\small
\centering
\caption{List of Acronyms used.}
\label{tab:acronyms}
\begin{tabular}{ll}
\hline
\textbf{Terms}               & \textbf{Acronyms} \\ \hline
Gradient Descent & GD \\ \hline
Dense-to-sparse learning & DSL \\ \hline
\end{tabular}
\end{table}

\section{BACKGROUND: DIFFERENTIABLE (A.K.A. SOFT) DECISION TREES}
\label{supp-sec:soft-trees}
A soft tree is a variant of a classical decision tree that performs soft routing, i.e., a sample is fractionally routed to all leaves.
It was proposed by \citet{Jordan1994}, and further developed in \citep{Kontschieder2015,Hazimeh2020b} for end-to-end optimization.
Soft routing makes soft trees differentiable, so learning can be done using gradient-based methods. 

Let us fix some $j \in [m]$ and consider a single tree $\B f^j$, which takes an input sample and returns an output vector (logit), i.e., $\B f^j: X \in \R^p \rightarrow \R^c$. 
Moreover, we assume that $\B f^j$ is a perfect binary tree with depth $d$. 
Let $\C I^j$ and $\C L^j$ denote the sets of internal (split) nodes and the leaves of the tree, respectively. 
For any node $i \in \C{I}^j \cup \C{L}^j$, we define $A^j(i)$ as its set of ancestors and use the notation {$\B x \rightarrow i$} for the event that a sample $\B x \in \R^p$ reaches $i$. See Table \ref{tab:notation} for detailed notation summary.

\noindent\textbf{Routing} Following existing work \citep{Kontschieder2015,Hehn2019,Hazimeh2020b}, we present routing in soft trees with a probabilistic model.
Although the sample routing is formulated with a probabilistic model, the final prediction of the tree $\B f$ is a deterministic function as it assumes an expectation over the leaf predictions.
According to this probabilistic model, internal (split) nodes in a soft tree perform soft routing, where a sample is routed left and right with different probabilities. 
Classical decision trees are modeled with either axis-aligned splits \citep{Breiman1984, Quinlan1993} or hyperplane (a.k.a. oblique) splits \citep {Murthy1994}. Soft trees are based on hyperplane splits, where the routing decisions rely on a linear combination of the features. 
Particularly, each internal node $i \in \C{I}^j$ is associated with a trainable weight vector $\B{w}_i^j \in \R^p$ that defines the node’s hyperplane split.   

Given a sample $\B x \in \R^p$, the probability that internal node $i$ routes $\B x$ to the left is defined by $S(\B{w}_i^j \cdot \B x)$.
Now we discuss how to model the probability that $\B x$ reaches a certain leaf $l$.
Let $[l \shortarrow{5} i]$ (resp. [$i \shortarrow{7} l$]) denote the event that leaf $l$ belongs to the left (resp. right) subtree of node $i \in \C{I}^j$.
The probability that $\B x$ reaches $l$ is given by $ P^j(\{x \rightarrow l\}) = \prod_{i \in A(l)} r_{i,l}^j(\B x)$, where $r_{i,l}^j(\B x)$ is the probability of node $i$ routing $\B x$ towards the subtree containing leaf $l$, i.e., $r_{i,l}^j(x) := S(\B{w}_i^j \cdot \B x)^{1[l \shortarrow{5} i]} \odot (1 - S(\B{w}_i^j \cdot \B x))^{1[i \shortarrow{7} l]}$. Let $S: R \rightarrow [0, 1]$ be an activation function. Popular choices for $S$ include logistic function \citep{Jordan1994, Kontschieder2015, Frosst2017, Tanno2019, Hehn2019} and Smooth-step function (for hard routing) \citep{Hazimeh2020b}.
Next, we define how the root-to-leaf probabilities can be used to make the final prediction of the tree.

\noindent\textbf{Prediction} As with classical decision trees, we assume that each leaf stores a learnable weight vector $\B{o}_l^j \in R^c$.
For a sample $\B{x} \in \R^p$, prediction of the tree is defined as an expectation over the leaf outputs, i.e., $\B f^j(\B x) = \sum_{l \in \C{L}^j} P^j(\{\B x \rightarrow l\})\B o_l^j$.

\subsection{Tree Ensemble Tensor Formulation}
\citet{Ibrahim2022} proposed a tensor formulation for modeling tree ensembles more efficiently, which can lead to faster training times than classical formulations~\citep{Kontschieder2015,Hazimeh2020b}: $\sim10\times$ on CPUs and $\sim20\times$ on GPUs. We use a similar tensor formulation, as  we discuss below. 
The internal nodes in the trees across the ensemble are jointly modeled as a ``supernode''.   
In particular, an internal node $i \in \C{I}^j$ at depth $d$ in all trees can be condensed together into a supernode $i \in \C{I}$.
Let $\B W_i \in \R^{p,m}$ be learnable weight matrix, where each $j$-th column of the weight matrix contains the learnable weight vector $\B w_i^j$ of the original $j$-th tree in the ensemble.
Similarly, the leaves in the trees across the ensemble are jointly modeled as a superleaf.
Let $\B O_l \in \R^{m,c}$ be the learnable weight matrix to store the leaf nodes, where each $j$-th row contains the learnable weight vector $\B o_l^j$ in the original $j$-th tree in the ensemble. The prediction of the tree ensemble is 
$\B f(\B x)= (\sum_{l \in \C{L}} \B O_l \odot \prod_{i\in A(l)}\B R_{i,l})  \cdot \B{1}_m,$
where $\odot$ denotes the element-wise product, $\B R_{i,l} = S(\B W_i \cdot \B x)^{1[l \shortarrow{5} i]} \odot (1-S(\B W_i \cdot \B x))^{1[i \shortarrow{7} l]} \in \R^{m,1}$ and the activation function $S$ is applied element-wise. $\B 1_m \in \R^m$ is a vector of ones that combines the predictions of trees in the ensemble. We denote all the hyperplane parameters across all the supernodes of the tree ensemble as a tensor $\bm{\C{W}} \in R^{p,m,|\C{I}|}$ and the parameters across all the superleaves as a tensor $\bm{\C{O}} \in R^{m,c,|\C{L}|}$.

\section{PROOF OF THEOREM~\ref{conv-thm}}\label{app-proof}
\noindent\textbf{Overview and Preliminaries.} 
Let us denote the training loss corresponding to the sample $n$ as
\begin{equation}
   \Phi_n(\bm{\C{O}},\bm{\C W}) = L( y_n,\B f(\rvx_n; \bm{\C {W}}, \bm{\C O})).
\end{equation}
We also use the notation $\sX$ to denote the set of all decision variables in the model, $\sX:=(\B{\C{W}},\B{\C{O}})$. For $j\in[m]$ and $i\in\C{I}^j$ and $t\in[p]$, we let $w^j_{i,t}$ be the $t$-th coordinate of $\B{w}_i^j$.

In this proof, we follow the general steps outlined below:
\begin{enumerate}
    \item First, we show that $\Phi_n$ is $\C{M}$-smooth for some $\C{M}>0$ only depending on the data and the constants appearing in Assumptions~\ref{sproperties},~\ref{lproperties} and~\ref{o-bounded}. That is, there exists $\C{M}>0$ such that for any two $\sX_1=(\B{\C W}_1,\B{\C O}_1),\sX_2=(\B{\C W}_2,\B{\C O}_2)$ with $\|\B{\C O}_1\|_2,\|\B{\C O}_2\|_2\leq B$, we have
\begin{equation}
    \|\nabla \Phi_n(\bm{\C{O}}_1,\bm{\C W}_1)-\nabla \Phi_n(\bm{\C{O}}_2,\bm{\C W}_2)\|_2\leq \C{M} \|\sX_1-\sX_2\|_2.
\end{equation}
    This will prove the descent property of the algorithm.
\item Next, we show that as long as $\lambda_2>0$, the sequence of solutions generated by the algorithm is bounded.
\item Finally, we show that $\Phi_n$ is semi-algebraic~\citep[Chapter 2]{dries_1998,attouch2013convergence} and therefore satisfies the Kurdyka–Łojasiewicz (KL) property~\citep{kl1,kl2}. This will complete the proof of convergence.
\end{enumerate}

Before continuing with the proof, we derive some results that will be useful. For notational convenience, we drop the sample index $n$ as our results will be true for all samples.

First, for $j\in[m]$, $i\in\C{I}^j$ and $l\in\C{L}^j$
\begin{equation}\label{app-f-partials}
    \begin{aligned}
    \frac{\partial f(\rvx; \bm{\C {W}}, \bm{\C O})}{\partial \bm{w}_i^j} & = \sum_{l\in\C L^j}o_l^j \frac{\partial P^j(\{\rvx \rightarrow l\})}{\partial \bm{w}_i^j}\\
    \frac{\partial f(\rvx; \bm{\C {W}}, \bm{\C O})}{\partial o_l^j} & = P^j(\{\rvx \rightarrow l\}).
    \end{aligned}
\end{equation}
Thus,
\begin{align}
    \frac{\partial \Phi}{\partial \bm{w}_i^j}&= L'( y, f(\rvx; \bm{\C {W}}, \bm{\C O}))\frac{\partial f(\rvx; \bm{\C {W}}, \bm{\C O})}{\partial \bm{w}_i^j}\nonumber \\
    & = L'( y, f(\rvx; \bm{\C {W}}, \bm{\C O}))\sum_{l\in\C L^j}o_l^j \frac{\partial P^j(\{\rvx \rightarrow l\})}{\partial \bm{w}_i^j}.
\end{align}
If $i\notin \C A(l)$, then $\frac{\partial P^j(\{\rvx \rightarrow l\})}{\partial \bm{w}_i^j}=\B 0$. Otherwise,
\begin{align}
    \frac{\partial P^j(\{\rvx \rightarrow l\})}{\partial \bm{w}_i^j}=\prod_{\substack{k \in \C A(l) \\ k\neq i}}r_{k,l}^j(\rvx)\frac{\partial r_{i,l}^j(\rvx)}{\partial \bm{w}_i^j}.\label{app-p-partial}
\end{align}
Moreover, by the definition of $r_{i,l}^j$, we have
\begin{equation}\label{app-r-partial}
    \frac{\partial r_{i,l}^j(\rvx)}{\partial \bm{w}_i^j} = \begin{cases} S'(\B{w}_i^j \cdot \rvx)\rvx &\mbox{if }~~~ l \shortarrow{5} i \\
    -S'(\B{w}_i^j \cdot \rvx)\rvx &\mbox{if }~~~ l \shortarrow{7} i. 
    \end{cases}
\end{equation}
In addition, for $j\in[m]$ and $l\in\C{L}^j$,
\begin{align}
    \frac{\partial \Phi}{\partial o_l^j}&= L'( y, f(\rvx; \bm{\C {W}}, \bm{\C O}))\frac{\partial f(\rvx; \bm{\C {W}}, \bm{\C O})}{\partial o_l^j}\nonumber \\
    & = L'( y, f(\rvx; \bm{\C {W}}, \bm{\C O})) P^j(\{\rvx \rightarrow l\}).
\end{align}
We define 
\begin{equation}\label{app-phi-def}
\begin{aligned}
\phi_{i,t}^j(\sX) &=  \frac{\partial \Phi}{\partial {w}_{i,t}^j},\\
\phi_{l}^j(\sX) &=  \frac{\partial \Phi}{\partial {o}_{l}^j}.
\end{aligned}
\end{equation}

Next, we state a few technical lemma that will be useful in our proof.
\begin{lemma}\label{app-partition-lemma}
Define $E_{i,+}^{j},E_{i,-}^{j}\subseteq \R^{p,m,|\C{I}|}\times \R^{m,|\C{L}|}$ for $i\in\C{I}^j,j\in[m]$:
\begin{equation}
\begin{aligned}
    E_{i,+}^{j} = \{\B{w}_{i}^j\cdot \rvx =\theta\},~~
    E_{i,-}^{j} = \{\B{w}_{i}^j\cdot \rvx =-\theta\}
\end{aligned}
\end{equation}
and 
$$\C{D} = \left(\left[\R^{p,m,|\C{I}|}\times \R^{m,|\C{L}|}\right]\bigcap \{\|\B{\C{O}}\|_2\leq B\}\right)\setminus \bigcup_{\substack{j\in[m]\\i\in\C{I}^j}} \left\{E_{i,+}^{j}\cup E_{i,-}^{j}\right\}$$
where $B$ is defined in Assumption~\ref{o-bounded}. If $\sX\in \C{D}$,
then $\Phi(\sX)$ is infinitely differentiable. Moreover, 
$\C{D}$
can be partitioned into finitely many subsets.
\end{lemma}
\begin{proof}
Note that $\Phi(\sX)$ may not have infinitely many derivatives only if $S(\B{w}_i^j\cdot\rvx)$ is not smooth for some $j\in[m],i\in\C{I}^j$. By the construction of $S(\cdot)$ from Assumption~\ref{sproperties}, the activation $S(x)$ function is not infinitely differentiable only for $x=\pm\theta$. As a result, if $\sX\in\C{D}$, all activation functions are smooth and therefore $\Phi(\sX)$ is infinitely differentiable. \\
Moreover, note that each $E_{i,\pm}^j$ is an affine space with codimension~1. Therefore, each $E_{i,\pm}^j$ partitions the space (excluding $E_{i,\pm}^j$) into two subsets $\{\B{w}_i^j\cdot \rvx>\pm\theta\},\{\B{w}_i^j\cdot \rvx<\pm\theta\}$. Therefore, all $E_{i,\pm}^j$ can partition $\C{D}$ into finitely many subsets, as there are finitely many of sets $E_{i,\pm}^j$. 
\end{proof}

\begin{lemma}\label{app-partials-bounded}
Suppose $\sX\in\C{D}$. Under the assumption of Theorem~\ref{conv-thm}, there exists a numerical constant $C>0$, only depending on the data and the constants appearing in the assumptions of the theorem, such that the following functions and their gradients are bounded by $C$:
\begin{equation}
    \begin{aligned}
        & r^j_{i,l}(\rvx),~j\in[m],l\in\C{L}^j, i\in\C{A}(l)\\
        & \frac{\partial r^{j}_{i,j}(\rvx)}{\partial w_{i,t}^j},~j\in[m],l\in\C{L}^j, i\in\C{A}(l), t\in[p] \\
        & P^j(\{\rvx \rightarrow l\}),~j\in[m], l\in\C{L}^j.
        \end{aligned}
\end{equation}
\end{lemma}
\begin{proof}
First, note that $r_{i,l}^j(\rvx)\in[0,1]$ and therefore $r_{i,l}^j(\rvx)$ is uniformly bounded. Moreover, from~(\ref{app-r-partial})
\begin{equation}
    \frac{\partial r_{i,l}^j(\rvx)}{\partial \bm{w}_i^j} = \begin{cases} S'(\B{w}_i^j \cdot \rvx)\rvx &\mbox{if }~~~ l \shortarrow{5} i \\
    -S'(\B{w}_i^j \cdot \rvx)\rvx &\mbox{if }~~~ l \shortarrow{7} i 
    \end{cases}
\end{equation}
with other partial derivatives of $r_{i,l}^j(\rvx)$ being zero. As a result, $r_{i,l}^j(\rvx)$ has uniformly bounded derivative, where the bound only depends on Assumption~\ref{sproperties} and the data. This is true as $S'$ is bounded by the assumption. This also shows that $\frac{\partial r^{j}_{i,l}(\rvx)}{\partial w_{i,t}^j}$ is uniformly bounded. Next,
\begin{equation}
    \frac{\partial^2 r_{i,l}^j(\rvx)}{\partial (\bm{w}_i^j)^2} = \begin{cases} S''(\B{w}_i^j \cdot \rvx)\rvx\rvx^T &\mbox{if }~~~ l \shortarrow{5} i \\
    -S''(\B{w}_i^j \cdot \rvx)\rvx\rvx &\mbox{if }~~~ l \shortarrow{7} i 
    \end{cases}
\end{equation}
which is similarly uniformly bounded by Assumption~\ref{sproperties}. Therefore, $\frac{\partial r^{j}_{i,j}(\rvx)}{\partial w_{i,t}^j}$ has a uniformly bounded gradient.\\
Finally, $P^j(\{\rvx \rightarrow l\})\in[0,1]$ and by~\eqref{app-p-partial}, the gradient of $P^j(\{\rvx \rightarrow l\})$ is the product of bounded functions, and therefore bounded.
\end{proof}

\begin{lemma}\label{app-fl-lemma}
Suppose $\sX\in\C{D}$. Under the assumption of Theorem~\ref{conv-thm}, there exists a numerical constant $C>0$, only depending on the data and the constants appearing in the assumptions of the theorem, such that the following functions and their gradients are bounded by $C$:
\begin{equation}
    \begin{aligned}
    & f(\rvx;\B{\C{W}},\B{\C{O}})\\
        & L'(y,f(\rvx;\B{\C{W}},\B{\C{O}}))
    \end{aligned}
\end{equation}
\end{lemma}
\begin{proof}
First, 
\begin{align}
    |f(\rvx;\B{\C{W}},\B{\C{O}})|& = \left\vert \sum_{j=1}^m\sum_{l\in\C{L}^j}o_l^jP^j(\{\rvx\rightarrow l\})  \right\vert \nonumber \\
    & \leq \sum_{j=1}^m\sum_{l\in\C{L}^j}|o_l^j| \nonumber \\
    & \leq m|\C{L}^1| B.
\end{align}
Moreover, from~\eqref{app-f-partials} and Lemma~\ref{app-partials-bounded}, each coordinate of the gradient of $f(\rvx;\B{\C{W}},\B{\C{O}})$ is finite summation and product of uniformly bounded functions, which is bounded.\\
Next, as $f(\rvx;\B{\C{W}},\B{\C{O}})$ is bounded, $L'({y},f(\rvx;\B{\C{W}},\B{\C{O}}))$ is bounded by Assumption~\ref{lproperties}. Moreover, 
$$\nabla L'({y},f(\rvx;\B{\C{W}},\B{\C{O}}) = L''({y},f(\rvx;\B{\C{W}},\B{\C{O}}))\nabla f(\rvx;\B{\C{W}},\B{\C{O}})$$
which is bounded as $L''$ is bounded by Assumption~\ref{lproperties} and $\nabla f(\rvx;\B{\C{W}},\B{\C{O}})$ is bounded as we showed above.
\end{proof}

\begin{lemma}\label{app-bounded-lemma}
There exists a numerical constant $M>0$, only depending on the data and constants introduced in Assumptions~\ref{sproperties},~\ref{lproperties} and~\ref{o-bounded}, such that if for $\sX_1,\sX_2$ and $\alpha\in(0,1)$, $\alpha \sX_1+(1-\alpha)\sX_2\in\C{D}$, then one has
$$|\phi_{i,t}^j(\sX_1)-\phi_{i,t}^j(\sX_2)|\leq M \|\sX_1-\sX_2\|_2,~t\in[p],j\in[m],i\in\C{I}^j$$
where $\phi_{i,t}^j$ is defined in~\eqref{app-phi-def}.
\end{lemma}
\begin{proof}
Note that by the definition of $\phi_{i,t}^j$,
\begin{align*}
   \phi_{i,t}^j(\sX) &=L'( y, f(\rvx; \bm{\C {W}}, \bm{\C O}))\frac{\partial f(\rvx;\B{\C{W}},\B{\C{O}})}{\partial w_{i,t}^j}
\end{align*}
which is the product of two bounded functions with bounded derivatives by Lemma~\ref{app-fl-lemma}. As a result, there exists a constant $M>0$ such that 
\begin{equation}\label{app-grad-m}
    \left\Vert\frac{\partial \phi^j_{i,t}(\sX)}{\partial\sX}\right\Vert_2\leq M.
\end{equation}
For $\alpha\in[0,1]$, let $\hat{\phi}_{i,t}^j(\alpha)=\phi_{i,t}^j((1-\alpha)\sX_1+\alpha\sX_2)$. Then, $\hat{\phi}_{i,t}^j(\alpha)$ is differentiable for $\alpha\in(0,1)$, $\hat{\phi}_{i,t}^j(0)=\phi_{i,t}^j(\sX_1)$ and $\hat{\phi}_{i,t}^j(1)=\phi_{i,t}^j(\sX_2)$. Moreover, by the chain rule,
\begin{align*}
    \frac{d \hat{\phi}^j_{i,t}(\alpha)}{d\alpha} &= \bigg\langle\frac{\partial \hat{\phi}^j_{i,t}(\alpha)}{\partial\left[ (1-\alpha)\sX_1+\alpha\sX_2\right]},\frac{\partial\left[ (1-\alpha)\sX_1+\alpha\sX_2\right]}{\partial \alpha}\bigg\rangle \\
    & = \bigg\langle(\sX_2-\sX_1),\frac{\partial \phi^j_{i,t}(\alpha)}{\partial\left[ (1-\alpha)\sX_1+\alpha\sX_2\right]}\bigg\rangle
\end{align*}
where $\langle\cdot,\cdot\rangle$ denotes the inner product. As a result, by the fundamental theorem of calculus, 
\begin{align*}
    \left\vert\hat{\phi}_{i,j}^j(\alpha) - \hat{\phi}_{i,j}^j(0) \right\vert & =\left\vert \int_{0}^{\alpha} \frac{d \hat{\phi}^j_{i,t}(u)}{du}du \right\vert  \\
    & \leq\|\sX_2-\sX_1\|_2 \int_{0}^{\alpha}\underbrace{\left\Vert  \frac{\partial \hat{\phi}^j_{i,t}(u)}{\partial\left[ (1-u)\sX_1+u\sX_2\right]} \right\Vert_2}_{\leq M~\text{by~\eqref{app-grad-m}}} du  \\
    & \leq \alpha M \|\sX_2-\sX_1\|_2.
\end{align*}
In particular, by setting $\alpha=1$ the proof is complete.
\end{proof}
\begin{lemma}\label{app-segmented}
Take any two $\sX_1,\sX_2$ such that $\|\B{\C{O}}_1\|_2,\|\B{\C{O}}_2\|_2\leq B$. Suppose for all $j\in[m],i\in\C{I}^j$, $\{\sX_1,\sX_2\}\not\subseteq E_{i,+}^j$ or $\{\sX_1,\sX_2\}\not\subseteq E_{i,-}^j$. Then $$|\phi_{i,t}^j(\sX_1)-\phi_{i,t}^j(\sX_2)|\leq M \|\sX_1-\sX_2\|_2~t\in[p],j\in[m],i\in\C{I}^j$$
where $\phi_{i,t}^j$ is defined in~\eqref{app-phi-def} and $M$ is defined in Lemma~\ref{app-bounded-lemma}.
\end{lemma}
\begin{proof}
By Lemma~\ref{app-partition-lemma}, the sets $E_{i,\pm}^j$ are affine and therefore if both $\sX_1,\sX_2$ do not belong to one of these sets, the segment connecting $\sX_1,\sX_2$ will intersect each set $E_{i,\pm}^j$ at most once. Thus, there exist $K\geq 1$ and $\tilde{\sX}_0,\cdots,\tilde{\sX}_K$ such that $\tilde{\sX}_0=\sX_1,\tilde{\sX}_K=\sX_2$, and $\tilde{\sX}_0,\cdots,\tilde{\sX}_K$ lie in the segment connecting $\sX_1,\sX_2$. Moreover, we have that 
$$\{(1-\alpha)\tilde{\sX}_k+\alpha\tilde{\sX}_{k+1}:\alpha\in(0,1)\}\subseteq \C{D}~~\forall k\in[K-1].$$
By triangle inequality, we have 
\begin{align*}
    |\phi_{i,t}^j(\sX_1)-\phi_{i,t}^j(\sX_2)|& = \left\vert\sum_{k=0}^{K-1}\{\phi_{i,t}^j(\tilde{\sX}_k)-\phi_{i,t}^j(\tilde{\sX}_{k+1})\}\right\vert  \\
    & \leq \sum_{k=0}^{K-1} \left\vert \phi_{i,t}^j(\tilde{\sX}_k)-\phi_{i,t}^j(\tilde{\sX}_{k+1}) \right\vert \\
    & \leq M\sum_{k=0}^{K-1} \|\tilde{\sX}_k-\tilde{\sX}_{k+1}\|_2\\
    & = M\|\sX_1-\sX_2\|_2
\end{align*}
where the last inequality is by Lemma~\ref{app-bounded-lemma} and the last equality is by the fact that $\tilde{\sX}_0,\cdots,\tilde{\sX}_K$ lie in the segment connecting $\sX_1,\sX_2$.
\end{proof}

\begin{lemma}\label{app-smooth-final}
There exists a numerical constant $\C{M}>0$, only depending on the data and constants introduced in Assumptions~\ref{sproperties},~\ref{lproperties} and~\ref{o-bounded}, such that $\Phi(\B{\C{O}},\B{\C{W}})$ is $\C{M}$-smooth for $\|\B{O}\|_2\leq B$.
\end{lemma}
\begin{proof}
First, note that
\begin{equation}\label{app-gradient-traingle}
    \|\nabla \Phi(\bm{\C{O}}_1,\bm{\C W}_1)-\nabla \Phi(\bm{\C{O}}_2,\bm{\C W}_2)\|_2 \leq  \sum_{j=1}^m\sum_{l\in \C{L}^j}|\phi^j_l(\sX_1)-\phi^j_l(\sX_2)| +
    \sum_{j=1}^m\sum_{i\in \C{I}^j}\sum_{t=1}^p|\phi^j_{i,t}(\sX_1)-\phi^j_{i,t}(\sX_2)|.
\end{equation}
Take $\sX_1, \sX_2$ with $\|\B{O}_1\|_2,\|\B{O}_2\|_2\leq B$. If these two solutions simultaneously do not belong to some $E_{i,\pm}^j$, by Lemma~\ref{app-segmented} we have
$$|\phi^j_{i,t}(\sX_1)-\phi^j_{i,t}(\sX_2)|\leq M \|\sX_1-\sX_2\|_2.$$
A similar result follows for $\phi^j_l$, hence by~\eqref{app-gradient-traingle} we achieve
$$ \|\nabla \Phi(\bm{\C{O}}_1,\bm{\C W}_1)-\nabla \Phi(\bm{\C{O}}_2,\bm{\C W}_2)\|_2\leq m(|\C{L}^1|+p|\C{I}^1||)M\|\sX_1-\sX_2\|_2$$
which completes the proof. Suppose there exists $i_0,j_0$ such that $\sX_1,\sX_2\in E^{j_0}_{i_0,+}$. Then, 
$$\left((1-\alpha)(\B{w}_1)_{i_0}^{j_0}\right)\cdot\rvx+\left(\alpha(\B{w}_2)_{i_0}^{j_0}\right)\cdot\rvx=\theta$$
therefore 
$S\left(\left((1-\alpha)(\B{w}_1)_{i_0}^{j_0}\right)\cdot\rvx+\left(\alpha(\B{w}_2)_{i_0}^{j_0}\right)\cdot\rvx\right)=1$
by Assumption~\ref{sproperties}. As a result, $S(\B{w}_{i_0}^{j_0}\cdot \rvx)=1$ for the whole segment connecting $\sX_1,\sX_2$. Let $\tilde{\sX}_1,\tilde{\sX}_2$ be such that all their weights are the same as $\sX_1,\sX_2$, except that
$$(\tilde{\B{w}}_1)_{i_0}^{j_0}=2({\B{w}}_1)_{i_0}^{j_0},(\tilde{\B{w}}_2)_{i_0}^{j_0}=2({\B{w}}_2)_{i_0}^{j_0}.$$
As a result,
$$\left((1-\alpha)(\tilde{\B{w}}_1)_{i_0}^{j_0}\right)\cdot\rvx+\left(\alpha(\tilde{\B{w}}_2)_{i_0}^{j_0}\right)\cdot\rvx=2\theta$$
so $S\left(\left((1-\alpha)(\tilde{\B{w}}_1)_{i_0}^{j_0}\right)\cdot\rvx+\left(\alpha(\tilde{\B{w}}_2)_{i_0}^{j_0}\right)\cdot\rvx\right)=1$. Therefore, one has 
\begin{equation*}
    \begin{aligned}
    {f}(\rvx;(1-\alpha)\tilde{\sX}_1+\alpha\tilde{\sX}_2)&={f}(\rvx;(1-\alpha)\sX_1+\alpha\sX_2),~\alpha\in[0,1],\\  
    \frac{\partial f(\rvx;(1-\alpha)\tilde{\sX}_1+\alpha\tilde{\sX}_2)}{\partial \B{w}_{i}^j}&=\frac{\partial f(\rvx;(1-\alpha)\sX_1+\alpha\sX_2)}{\partial \B{w}_{i}^j},~\alpha\in[0,1],j\in[m],i\in\C{I}^j
    \end{aligned}
\end{equation*}
as for all $i,j$,
$$S\left(\left((1-\alpha)(\tilde{\B{w}}_1)_{i}^{j}\right)\cdot\rvx+\left(\alpha(\tilde{\B{w}}_2)_{i}^{j}\right)\cdot\rvx\right)=S\left(\left((1-\alpha)({\B{w}}_1)_{i}^{j}\right)\cdot\rvx+\left(\alpha({\B{w}}_2)_{i}^{j}\right)\cdot\rvx\right).$$
As a result, for all $i,j,t$,
$$\phi^j_{i,t}(\tilde{\sX}_1)=\phi^j_{i,t}({\sX}_1),\phi^j_{i,t}(\tilde{\sX}_2)=\phi^j_{i,t}({\sX}_2)$$
However, $\tilde{\sX}_1,\tilde{\sX}_2\notin E_{i_0,+}^{j_0}$. In words, the new points effectively replace the problematic coefficient $\B{w}_{i_0}^{j_0}$ in the model coefficients. Repeat this process until all such coefficients are removed and therefore $\tilde{\sX}_1,\tilde{\sX}_2$ fit into the assumptions of Lemma~\ref{app-segmented}. This completes the proof.
\end{proof}

\begin{lemma}\label{app-s-semi}
Under Assumption~\ref{sproperties}, the activation function $S(x)$ is semi-algebraic.
\end{lemma}
\begin{proof}
Consider the epigraph of the activation function
$$\C{G}(S)=\{(x,y):y\geq S(x)\}.$$
Let 
\begin{equation*}
    \begin{aligned}
    A_1&=\{y\geq 0, x\leq -\theta\}\\
    A_2&=\{y\geq 1, x\geq \theta\}\\
    A_3&=\{y\geq p(x), -\theta\leq x\leq \theta\}.
    \end{aligned}
\end{equation*}
Then, 
$$\C{G}(S)=A_1\cup A_2\cup A_3$$
showing $\C{G}(S)$ and consequently, $S$ are semi-algebraic.
\end{proof}
\begin{lemma}\label{app-indication-semi}
The function $1[\bm{\C{W}}_{k,:,:} \neq \B 0]$ is semi-algebraic.
\end{lemma}
\begin{proof}
Consider the epigraph:
$$G=\C{G}(1[\bm{\C{W}}_{k,:,:} \neq \B 0])=\{(y,\B{\C{W}}):y\geq 1[\bm{\C{W}}_{k,:,:} \neq \B 0]\}.$$
Let 
$A_1 = \{y\geq 1\}$
and 
$A_2=\{\|\bm{\C{W}}_{k,:,:}\|_2^2=0,y\geq 0\}.$
Then, $G=A_1\cup A_2$ showing $1[\bm{\C{W}}_{k,:,:} \neq \B 0]$ is semi-algebraic.
\end{proof}
\begin{proof}[\textit{Proof of Theorem~\ref{conv-thm}}.]
\textbf{Part 1) } Note that by Assumption~\ref{o-bounded}, Algorithm~\ref{algo:proximal-stochastic-gradient-descent} can be equivalently run on the problem
\begin{equation}\label{app-constrained}
        {\min}_{\bm{\C {W}}, \bm{\C O}}~~\hat{\E}[L(\rvy, \B f( \rvx; \bm{\C {W}}, \bm{\C O}))] + \lambda_0 {\sum}_{k \in [p]} 1[\bm{\C{W}}_{k,:,:} \neq \B 0] + (\lambda_2/m|\C{I}|) \|\bm{\C{W}}\|_2^2 ~~~~~~~\text{s.t.} ~~~~~~~\|\B{\C{O}}\|_2\leq B
\end{equation}
and result in the same sequence of solutions. Let 
\begin{equation}\label{app-g-h}
    \begin{aligned}
    h(\bm{\C {W}}, \bm{\C O}) & = \hat{\E}[L(\rvy, \B f( \rvx; \bm{\C {W}}, \bm{\C O}))] + (\lambda_2/m|\C{I}|) \|\bm{\C{W}}\|_2^2 \\
     g(\bm{\C {W}}, \bm{\C O}) & =  \lambda_0 {\sum}_{k \in [p]} 1[\bm{\C{W}}_{k,:,:} \neq \B 0]+\chi(\|\B{\C{O}}\|_2\leq B)
    \end{aligned}
\end{equation}
where $\chi(\|\B{\C{O}}\|_2\leq B)=0$ if 
$\|\B{\C{O}}\|_2\leq B$ and $\chi(\|\B{\C{O}}\|_2\leq B)=\infty$ otherwise. Then, Problem~\ref{app-constrained} can be written as minimizing $g+h$. By Lemma~\ref{app-smooth-final}, the function $h$ is $(\C{M}+2\lambda_2/m|\C{I}|)$-smooth. Hence, if $\eta<1/(\C{M}+2\lambda_2/m|\C{I}|)$, by Lemma~3.1 of~\citet{attouch2013convergence} (and calculations leading to (52) of~\citet{attouch2013convergence}) the descent property is proved.

\textbf{Part 2) } Let $c_0$ be the objective value for the initial solution to Algorithm~\ref{algo:proximal-stochastic-gradient-descent}. That is, 
\begin{equation*}
   \frac{1}{N} \sum_{n=1}^N L(y_n, \B f( \rvx_n; \bm{\C {W}}, \bm{\C O}))] + \lambda_0 {\sum}_{k \in [p]} 1[\bm{\C{W}}_{k,:,:} \neq \B 0]\\ + (\lambda_2/m|\C{I}|) \|\bm{\C{W}}\|_2^2=c_0.
\end{equation*}
As the sequence of the objectives is non-increasing by the first part of the theorem, at each iteration we have
\begin{align*}
   & (\lambda_2/m|\C{I}|) \|\bm{\C{W}}\|_2^2  \\ 
   \leq &  \frac{1}{N}\sum_{n=1}^N L(y_n, \B f( \rvx_n; \bm{\C {W}}, \bm{\C O}))] + \lambda_0 {\sum}_{k \in [p]} 1[\bm{\C{W}}_{k,:,:} \neq \B 0] + (\lambda_2/m|\C{I}|) \|\bm{\C{W}}\|_2^2 \\
   \leq & c_0
\end{align*}
completing the proof as $\lambda_2>0$.

\textbf{Part 3) } Note that the function $f(\rvx;\B{\C{W}},\B{\C{O}})$ is a composition, finite sum and finite product of semi-algebraic functions by Lemma~\ref{app-s-semi}, therefore it is semi-algebraic. As $L$ is polynomial, $L(f(\rvx;\B{\C{W}},\B{\C{O}}))$ is semi-algebraic. As a result, by Lemma~\ref{app-indication-semi} the functions $g,h$ defined in~\eqref{app-g-h} are semi-algebraic and hence, possess the KL property. Therefore, $g,h$ satisfy the required conditions of Theorem~5.1 of~\citet{attouch2013convergence}, proving the convergence.
\end{proof}

\section{APPENDIX FOR SECTION \ref{sec:simulation}}
\label{supp-sec:simulation}
\noindent\textbf{Data Design}
We generate the data matrix, $\B X\in\R^{N \times p}$ with samples randomly drawn from a multivariate normal distribution $\C{N}(0, \B \Sigma)$ with a correlated design matrix $\B \Sigma$ whose values are defined by $\Sigma_{ij}=\sigma^{|i-j|}$.
We construct the response variable $\B y=\B X\B\beta^{*}+\epsilon$ where the values of the noise $\epsilon_i, i=[N]$ are drawn independently from $\C{N}(0, 0.5)$. 
We use a known sparsity $\B\beta^{*}$ with equi-spaced nonzeros, where $\norm{\B\beta^{*}}_0=8$.  

\noindent\textbf{Simulation Procedure} We experiment with a range of training set sizes $N \in \{100, 200, 1000\}$, correlation strengths $\sigma \in \{0.5, 0.7\}$, and number of total features $p \in \{256,512\}$.
For each setting, we run 25 simulations with randomly generated samples.
We evaluate on $10,000$ test samples drawn from the data generating model described above.
Out of $N$ samples, we allocate $80\%$ for training and $20\%$ for validation for model selection.
For each simulation, we perform 500 tuning trials (hyperparameters are given below) and select the model with the smallest validation mean squared error (MSE).
We evaluate the final performance in terms of (i) test MSE, (ii) number of features selected and (iii) support recovery via f1-score between the true support and the recovered feature set.
We compute averages and standard errors across the 25 simulations to report final results.

\noindent\textbf{\modelname~with DSL}
\begin{itemize}[noitemsep,topsep=0pt,parsep=0pt,partopsep=0pt, leftmargin=*]
\item Number of trees: Discrete uniform with range $[1, 50]$,
\item Depths: Discrete uniform with range $[1, 5]$, 
\item Batch sizes: $16b$ with b uniform over the range $\{1, 2, \ldots, 8\}$,
\item Number of Epochs: Discrete uniform with range $[5, 500]$,
\item $\lambda_0$: exponentially ramped up $0 \rightarrow 100$ with temperature distributed as Log uniform in the range $[10^{-4}, 0.1]$ for group $\ell_0-\ell_2$ with DSL,
\item $\lambda_2$: Log uniform over the range $[10^{-2}, 10^2]$ for group $\ell_0-\ell_2$,
\item Learning rates, $lr$: Log uniform over the range $[10^{-3}, 10^{-1}]$.
\end{itemize}

\noindent\textbf{XGBoost, Random Forests}
\begin{itemize}[noitemsep,topsep=0pt,parsep=0pt,partopsep=0pt, leftmargin=*]
\item Number of trees: Discrete uniform with range $[1, 100]$ for XGBoost and Random Forests,
\item Depths: Discrete uniform with range $[1, 10]$ for XGBoost and Random Forests, 
\item Learning rates: Discrete uniform with range $[10^{-4}, 1.0]$ for XGBoost, 
\item Feature importance threshold: Log uniform over the range $[10^{-7}, 10^{-1}]$ for XGBoost and Random Forests.
\item Subsample: Uniform over the range $[0.5, 1.0]$.
\end{itemize}

\section{APPENDEX FOR SECTION \ref{sec:experiments}}
\label{supp-sec:experiments}

\subsection{Datasets, Computing Setup and Tuning Setup} 

\begin{wraptable}[19]{r}{0.3\textwidth}
\centering
\caption{Classification datasets}
\label{tab:classification-datasets}
\setlength{\tabcolsep}{4.0pt}
\begin{tabular}{l|ccc}
Dataset & N & p & C \\ \hline
Churn & 5,000 & 20 & 2 \\ 
Satimage & 6,435 & 36 & 6 \\ 
Texture & 5,500 & 40 & 11 \\ 
Mice-protein & 1,080 & 77 & 8 \\ 
Isolet & 7,797 & 617 & 26 \\ 
Madelon & 2,600 & 500 & 2 \\ 
Lung & 203 & 3,312 & 5 \\ 
Gisette & 7,000 & 5,000 & 2 \\ 
Tox & 171 & 5,748 & 4 \\ 
Arcene & 200 & 10,000 & 2 \\ 
CLL & 111 & 11,340 & 3 \\ 
SMK & 187 & 19,993 & 2 \\ 
GLI & 85 & 22,283 & 2 \\ 
Dorothea & 1,150 & 100,000 & 2 \\ \hline 
\end{tabular}
\end{wraptable}
\paragraph{Datasets} We consider 5 classification datasets from the Penn Machine Learning Benchmarks (PMLB) repository \citep{Olson2017}.
These are Churn, Satimage, Mice-protein, Isolet and Texture; Mice-protein and Isolet were used in prior feature selection literature \cite{Lemhadri2021}. We used fetch\_openml in Sklearn to download the full datasets.
We used 4 datasets from NIPS2003 feature selection challenge \citep{Guyon2004}. These are Arcene, Madelon, Gisette, Dorothea.
We used 5 datasets (Smk, Cll, Gli, Lung, Tox) from the feature selection datasets given in this repo\footnote{\url{https://jundongl.github.io/scikit-feature/datasets.html}}.
The datasets contain continuous, categorical and binary features. The metadata was used to identify the type of features and categorical features were one-hot encoded.
For first non-NIPS2003 datasets, we randomly split each of the dataset into 64\% training, 16\% validation and 20\% testing sets. 
For the datasets from NIPS2003, we split the training set into 80\% training and 20\% validation and treated the original validation set as the test set. The labels for original test set were unavailable, hence we discarded the original test set. 
The continuous features in all datasets were z-normalized based on training set statistics. A summary of the 14 datasets considered is in Table \ref{tab:classification-datasets}.

\noindent\textbf{Computing Setup} We used a cluster running Ubuntu 7.5.0 and equipped with Intel Xeon
Platinum 8260 CPUs and Nvidia Volta V100 GPUs. For all experiments of Sec. \ref{sec:simulation} and \ref{sec:experiments}, each job involving \modelname, Random Forests, XGBoost, LightGBM,  CatBoost and neural networks were restricted to 4 cores and 64GB of RAM. Jobs involving \modelname~and neural networks on larger datasets (Gisette, Dorothea) were run on Tesla V100 GPUs. 

\noindent\textbf{Tuning}
The tuning was done in parallel over the competing models and datasets.
The number of selected features affects the AUC.
Therefore, to treat all the methods in a fair manner, we tune the hyperparameter that controls the sparsity level using Optuna \citep{Akiba2019} which optimizes the overall AUC across different Ks.
A list of all the tuning parameters and their distributions is given for every experiment below.

\subsection{Appendix for Sec. \ref{sec:comparison-with-tao}}
\label{supp-sec:comparison-with-tao}
\noindent\textbf{Classical tree: TAO Implementation} 
Given that the authors of TAO do not open-source their implementation, we have written our own implementation of the TAO algorithm proposed by \cite{Carreira-Perpinan2018}, following the procedure outlined in Sec. 3.1 of \cite{Zharmagambetov2019}. 
For TAO, we use oblique (i.e. hyperplane splits) decision trees with constant leaves. 
We take as an initial tree a complete binary tree with random parameters at each node. We perform TAO updates until maximum number of iterations are reached (i.e. there is no other stopping criterion).
TAO uses an $\ell_1$-regularized logistic regression to solve the decision node optimization (using LIBLINEAR \cite{Fan2008}) where $\lambda \geq 0$ parameter (controlling node-level sparsity of the tree) is used as a regularization parameter ($C = 1/\lambda$).
We tune depth in the range $[1,10]$, $\lambda \in [10^{-5}, 10^{5}]$ and number of maximum iterations in the range $[20, 100]$. We perform 500 tuning trials. We find the optimal trial that satisfies $50\%$ sparsity budget.

\noindent\textbf{Soft tree with group lasso}
We compare group $\ell_0$-based method with a competitive benchmark: the convex group lasso regulurizer, popularly used in high-dimensional statistics literature \citep{Yuan2006}.
We consider group-lasso regularization in the context of soft trees: $(\lambda_1/\sqrt{m|\C{I}|}) {\sum}_{k \in [p]} \norm{\bm{\C{W}}_{k,:,:}}_2.$
Although group lasso has not been used in soft trees, it has been used for feature selection in neural networks \citep{Scardapane2017}. 
However, their proposal to use GD on a group $\ell_1$ regularized objective does not lead to feature selection as highlighted by \citep{Lemhadri2021}.
For a fairer comparison, we implement our own proximal mini-batch GD method for group lasso in the context of soft trees, which actually leads to feature selection. 

We consider the following optimization problem with group-lasso regularization in the context of soft trees:
\begin{align}
    {\min}_{\bm{\C {W}}, \bm{\C O}} ~~\hat{\E}[L(\rvy, \B f( \rvx; \bm{\C {W}}, \bm{\C O})]  + (\lambda_1/\sqrt{m|\C{I}|}) {\sum}_{k \in [p]} \norm{\bm{\C{W}}_{k,:,:}}_2.
    \label{eq:group-lasso}
\end{align}
where $\lambda_1$ is a non-negative regularization parameter that controls both shrinkage and sparsity. We solve it with the algorithm presented in Alg. \ref{algo:proximal-stochastic-gradient-descent-group-lasso}. 
\begin{algorithm}[!h]
\small
\begin{algorithmic}[1]
\Require Data: $X, Y$; ~~~Hyperparameters: $\lambda_1$, epochs, batch-size ($b$), learning rate ($\eta$); 
\State Initialize ensemble with $m$ trees of depth $d$ ($|\C{I}|=2^d - 1$): $\bm{\C{W}}, \bm{\C{O}}$ 
\For {$\text{epoch}=1,2,\ldots,\text{epochs}$}
\For {$batch=1,\ldots,N/b$}
\State Randomly sample a batch: $\B X_{batch}, \B Y_{batch}$ 
\State Compute gradient of loss $g$ w.r.t. $\bm{\C{O}}$, \bm{\C{W}}, where $g=\hat{\E}[L(\B Y_{batch}, \B F_{batch}]$.
\State Update leaves: $\bm{\C{O}} \gets \bm{\C{O}} - \eta \nabla_{\bm{\C{O}}}g$
\State Update hyperplanes: $\bm{\C{W}} \gets \textit{S-Prox}(\bm{\C{W}} - \eta \nabla_{\bm{\C{W}}}g, \eta, \lambda_1))$
\EndFor
\EndFor
\end{algorithmic}
\caption{Proximal Mini-batch Gradient Descent for Optimizing (\ref{eq:group-lasso}).}
\label{algo:proximal-stochastic-gradient-descent-group-lasso}
\end{algorithm}

\textit{S-Prox} in Algorithm \ref{algo:proximal-stochastic-gradient-descent-group-lasso} finds the global minimum of an optimization problem of the form:
\begin{align}
\B{\C{W}}^{(t)} = \text{argmin}_{\B{\C{W}}} \frac{1}{2\eta} \norm{\B{\C{W}} - \B{\C{Z}}^{(t)}}_2^2 +  \frac{\lambda_1}{\sqrt{m|\C{I}|}} \sum_{k=1}^p \norm{\bm{\C{W}}_{k,:,:}}_2 
\end{align}
where $\B{\C{Z}}^{(t)} = \B{\C{W}}^{(t-1)}-\eta\nabla_{\bm{\C{W}}}g$ and $g = \hat{\E}[L(\B Y_{batch}, \B F_{batch}]$. This leads to a feature-wise soft-thresholding operator given by:
\begin{align}
S_{\frac{\eta\lambda_1}{\sqrt{m|\C{I}|}}}(\B{\C{Z}}_{k,:,:}^{(t)}) = \begin{cases}
\B{\C{Z}}_{k,:,:}^{(t)}- \frac{\eta\lambda_1}{\sqrt{m|\C{I}|}} \frac{\B{\C{Z}}_{k,:,:}^{(t)}}{\norm{\B{\C{Z}}_{k,:,:}^{(t)}}} &\text{if $\norm{\B{\C{Z}}_{k,:,:}^{(t)}} \geq \frac{\eta\lambda_1}{\sqrt{m|\C{I}|}}$}\\
0 &\text{otherwise}
\end{cases}
\end{align}

In these experiments, we used a single tree and mini-batch PGD \textit{without} any dense-to-sparse scheduler for both models i.e, (i) Soft tree with group lasso (ii) \modelsname.
We tune the key hyperparameters, which are given below.
\begin{itemize}[noitemsep,topsep=0pt,parsep=0pt,partopsep=0pt, leftmargin=*]
\item Batch sizes: $64*b$ with $b$ uniform over the range $\{1, 2, \ldots, 64\}$,
\item Learning rates for mini-batch PGD: Log uniform over the range $[10^{-2}, 10]$,
\item Number of Epochs: Discrete uniform with range $[5, 1000]$,
\item $\lambda_0$: Log uniform over the range $[10^{-3}, 10]$ for group $\ell_0-\ell_2$ for \modelsname, 
\item $\lambda_2$: Log uniform over the range $[10^{-2}, 10^2]$ for group $\ell_0-\ell_2$ for \modelsname,
\item $\lambda_1$: Log uniform over the range $[10^{-3}, 10]$ for group lasso.
\end{itemize}

\subsection{Appendix for Sec. \ref{sec:results-dense-vs-sparse} and \ref{sec:results-soft-vs-classical}}
\label{supp-sec:soft-vs-classical}
We describe the tuning grid for these experiments below.

\noindent\textbf{\modelname}
\begin{itemize}[noitemsep,topsep=0pt,parsep=0pt,partopsep=0pt, leftmargin=*]
\item Number of trees: Discrete uniform with range $[1, 100]$,
\item Depths: Discrete uniform with range $[1, 5]$, 
\item Batch sizes: Uniform over the set $\{16, 64, 256, 1024\}$,
\item Number of Epochs: Discrete uniform over the range $[5, 1000]$ (Note tuning over epochs is important to achieve some trials with desired sparsity for dense-to-sparse learning, we do not use any validation loss early stopping as that is less robust during averages across runs/seeds in terms of feature support.),
\item $\lambda_0$: exponentially ramped up $0 \rightarrow 1$ with temperature distributed as Log uniform in the range $[10^{-4}, 1]$ for group $\ell_0-\ell_2$ with DSL,
\item $\lambda_2$: Log uniform over the range $[10^{-3}, 1]$ for group $\ell_0-\ell_2$,
\item Learning rate ($lr$) for minibatch PGD: Log uniform over the range $[10^{-2}, 10]$.
\end{itemize}

\noindent\textbf{XGBoost, LightGBM, CatBoost, Random Forests}
\begin{itemize}[noitemsep,topsep=0pt,parsep=0pt,partopsep=0pt, leftmargin=*]
\item Number of trees: Discrete uniform with range $[1, 300]$ for XGBoost, LightGBM, CatBoost and Random Forests,
\item Depths: Discrete uniform with range $[1, 10]$ for XGBoost, LightGBM, CatBoost and Random Forests, 
\item Learning rates: Discrete uniform with range $[10^{-4}, 1.0]$ for XGBoost, LightGBM and CatBoost, 
\item Feature importance threshold: Log uniform over the range $[10^{-7}, 10^{-1}]$ for XGBoost, LightGBM, CatBoost and Random Forests,
\item Subsample: Uniform over the range $[0.5, 0.9]$ for XGBoost, LightGBM, CatBoost and Random Forests.
\item Bagging Frequency: Uniform over the set $\{1, \cdots, 7\}$ for LightGBM.
\end{itemize}

\noindent\textbf{FASTEL (Dense soft trees)}
\begin{itemize}[noitemsep,topsep=0pt,parsep=0pt,partopsep=0pt, leftmargin=*]
\item Number of trees: Discrete uniform with range $[1, 100]$,
\item Depths: Discrete uniform with range $[1, 5]$, 
\item Batch sizes: Uniform over the set $\{16, 64, 256, 1024\}$,
\item Number of Epochs: Discrete uniform over the range $[5, 1000]$,
\item Learning rates ($lr$) for minibatch SGD: Log uniform over the range $[10^{-2}, 10]$.
\end{itemize}

\subsection{Appendix for Sec. \ref{sec:results-soft-vs-nn}}
\label{supp-sec:soft-vs-neural-networks}
In this paper, we pursue embedded feature selection methods for tree ensembles.
However, for completeness, we also compare \modelname~ against some state-of-the-art embedded feature selection methods from neural networks.

\noindent\textbf{Competing Methods}
We compare against the following baselines.
\begin{enumerate}
    \item LassoNet \citep{Lemhadri2021} is based on a ResNet-like \citep{He2016} architecture with residual connections. Feature selection is done using hierarchical group lasso.
    \item AlgNet \citep{Vu2020} is based on adaptive lasso and uses proximal full-batch gradient descent for feature selection in a tanh-activated feedforward network. 
    \item DFS \citep{Chen2021} solves cardinality constrained feature selection problem with an active-set style method.
\end{enumerate}

We describe the tuning details for these neural network models below. We perform 2000 tuning trials for each method.
For DFS, we capped number of trials completed in total 200 GPU hours. DFS is really slow for even medium sized datasets (in terms of feature dimensions).

\paragraph{LassoNet \citep{Lemhadri2021}}
\begin{itemize}[noitemsep,topsep=0pt,parsep=0pt,partopsep=0pt, leftmargin=*]
\item ResNet architecture with 2-layered relu-activated feedforward network with (linear) skip connections.
\item Number of hidden units: Discrete uniform in the set $\{\frac{p}{3}, \frac{2}{3}p, p, \frac{4}{3}p\}$,
\item Batch sizes: Discrete uniform in the set $\{64, 128, 256, 512\}$,
\item $\lambda$: Log uniform over the range $[1, 10000]$.
\item Tuning protocol: $1000$ for dense training stage and $1000$ for each successive sparse training stages with early stopping with a patience of 25. We consider 100 sequential stages (100 values for $\lambda$).
\end{itemize}

\noindent\textbf{AlgNet \citep{Vu2020}}
\begin{itemize}[noitemsep,topsep=0pt,parsep=0pt,partopsep=0pt, leftmargin=*]
\item Tanh-activated 4-layered feedforward neural network with number of hidden units chosen uniformly from a  discrete set $\{\frac{p}{3}, \frac{2}{3}p, p, \frac{4}{3}p\}$,
\item Learning rates ($lr$) for proximal GD: Log uniform over the range $[0.01, 1]$.
\item $\lambda$: Log uniform over the range $[0.1, 1000]$.
\item Number of Epochs: Discrete uniform with range $[5, 2000]$,
\end{itemize}

\noindent\textbf{DFS \citep{Chen2021}}
\begin{itemize}[noitemsep,topsep=0pt,parsep=0pt,partopsep=0pt, leftmargin=*]
\item ReLU-activated 4-layered feedforward neural network with number of hidden units chosen uniformly from a discrete set $\{\frac{p}{3}, \frac{2}{3}p, p, \frac{4}{3}p\}$,
\item Learning rates ($lr$) for Adam: Log uniform over the range $[0.001, 0.1]$.
\item Weight decay: $0.0025$ for feature selection layer and $0.005$ for remaining layers.
\item $k$: Number of features selected in the range $[1, 0.5p]$.
\item Number of Epochs: $500$ with early stopping.
\end{itemize}

\subsection{Appendix for Sec. \ref{sec:results-dsl}}
\label{supp-sec:schedulers}
We perform 2000 tuning trials for each method.

\noindent\textbf{\modelname}
\begin{itemize}[noitemsep,topsep=0pt,parsep=0pt,partopsep=0pt, leftmargin=*]
\item Number of trees: Discrete uniform with range $[1, 100]$,
\item Depths: Discrete uniform with range $[1, 5]$, 
\item Batch sizes: Uniform over the set $\{16,64,256,1024\}$,
\item Number of Epochs: Discrete uniform with range $[5, 1000]$,
\item $\lambda_0$ (without Dense-to-sparse learning): Log uniform in the range $[1, 10^{4}]$ for group $\ell_0-\ell_2$,
\item $\lambda_0$ (with Dense-to-sparse learning): exponentially ramped up $0 \rightarrow 1$ with temperature $s$ distributed as Log uniform in the range $[10^{-4}, 1]$ for group $\ell_0-\ell_2$ with DSL,
\item $\lambda_2$: Log uniform over the range $[10^{-3}, 1]$ for group $\ell_0-\ell_2$,
\item Learning rates ($lr$): Log uniform over the range $[10^{-2}, 10]$.
\end{itemize}

\begin{wraptable}[13]{r}{0.5\textwidth}
\small
\centering
\captionsetup{width=.5\textwidth}
\caption{Comparison of test AUC (\%) performance of \modelname~against ControlBurn for $25\%$ feature selection budget on binary classification tasks.}
\label{tab:controlburn}
\setlength{\tabcolsep}{20pt}
\resizebox{0.5\textwidth}{!}{\begin{tabular}{l|c|c}
\multicolumn{1}{l|}{\multirow{1}{*}{Dataset}} & \multirow{1}{*}{ControlBurn} & \multicolumn{1}{c}{\modelname} \\ \hline
\multicolumn{1}{l|}{Churn}  & \multicolumn{1}{l|}{85.66$\pm$0.25} & \multicolumn{1}{l}{\textbf{91.38}$\pm$0.08} \\ 
\multicolumn{1}{l|}{Madelon}  & \multicolumn{1}{l|}{94.23$\pm$0.08} & \multicolumn{1}{l}{94.14$\pm$0.09} \\ 
\multicolumn{1}{l|}{Gisette}  & \multicolumn{1}{l|}{99.36$\pm$0.01} & \multicolumn{1}{l}{\textbf{99.81}$\pm$0.002} \\ 
\multicolumn{1}{l|}{Arcene}  & \multicolumn{1}{l|}{77.89$\pm$0.32} & \multicolumn{1}{l}{\textbf{80.80}$\pm$0.30} \\ 
\multicolumn{1}{l|}{Smk}  & \multicolumn{1}{l|}{\textbf{79.94}$\pm$0.32} &  \multicolumn{1}{l}{79.29$\pm$0.22} \\ 
\multicolumn{1}{l|}{Gli}  & \multicolumn{1}{l|}{98.62$\pm$0.16} & \multicolumn{1}{l}{\textbf{99.80}$\pm$0.07} \\ 
\multicolumn{1}{l|}{Dorothea}  & \multicolumn{1}{l|}{84.85$\pm$0.28} & \multicolumn{1}{l}{\textbf{90.87}$\pm$0.02} \\ 
\hline
\multicolumn{1}{l|}{Average} & \multicolumn{1}{l|}{88.65} & \multicolumn{1}{l}{\textbf{90.87}}
\end{tabular}}
\end{wraptable}
\subsection{Comparison with ControlBurn}
\label{supp-sec:controlburn}
We also compare against ControlBurn \citep{Liu2021} on binary classification tasks. ControlBurn does not support multiclass classification.
ControlBurn is another post-training feature selection method, which formulates an optimization problem with group-lasso regularizer to perform feature selection on a pre-trained forest. It relies on a commercial solver to optimize their formulation.
We report the results for $25\%$ feature selection budget in Table \ref{tab:controlburn}.
We can observe that \modelname~outperforms ControlBurn across many datasets, achieving $2\%$ (up to $6\%$) improvement in AUC.

We note the hyperparameter tuning for ControlBurn below:
\begin{itemize}[noitemsep,topsep=0pt,parsep=0pt,partopsep=0pt, leftmargin=*]
\item Method: bagboost,
\item Depths: Discrete uniform with range $[1, 10]$, 
\item $\alpha$: Log uniform over the range $[10^{-7}, 1.0]$.
\end{itemize}

\section{MEASURING STATISTICAL SIGNIFICANCE}
We follow the following procedure to test the significance of all models.
For all models, we tune over hyperparameters for 2000 trials. We select the optimal trial (within the desired feature sparsity budget) based on validation set. 
Next, we train each model for 100 repetitions with the optimal hyperparameters and report the mean results on test set along with the standard errors.

\end{document}